\documentclass[runningheads]{llncs}
\usepackage[T1]{fontenc}

\usepackage{url}
\usepackage{multirow}
\usepackage{amssymb}
\usepackage{amsxtra}
\usepackage{makeidx}
\usepackage{amsfonts}
\usepackage{mathtools}
\usepackage{algorithm}
\usepackage{algorithmic}
\usepackage{float}
\usepackage{graphicx}
\usepackage{comment}
\usepackage[font=small,labelfont=bf]{caption}
\usepackage{subcaption}
\usepackage{breqn}
\usepackage[english]{babel}
\usepackage{ulem}
\usepackage[graphicx]{realboxes}
\usepackage{color}
\usepackage{hyperref}

\begin{document}

\title{Mirror Descent-Type Algorithms for the Variational Inequality Problem with Functional Constraints}
\titlerunning{Mirror Descent Methods for VIs with Functional Constraints}

\author{Mohammad S. Alkousa\inst{1}\orcidID{0000-0001-5470-0182} \and
Fedor S. Stonyakin\inst{2,3,1}\orcidID{0000-0002-9250-4438} \and Belal A. Alashqar\inst{2}\orcidID{0009-0002-2082-7198} \and Seydamet S. Ablaev\inst{3}\orcidID{0000-0002-9927-6503}
}

\authorrunning{M. Alkousa et al.}

\institute{Innopolis University, Innopolis, Universitetskaya Str., 1, 420500, Russia. \and Moscow Institute of Physics and Technology, 9 Institutsky lane, Dolgoprudny, 141701, Russia.\and
V.\,I.\,Vernadsky Crimean Federal University, 4 Academician Vernadsky Avenue, Simferopol, 295007, Republic of Crimea, Russia.\\
\email{m.alkousa@innopolis.ru, fedyor@mail.ru, alashkar.ba@phystech.edu, seydamet.ablaev@yandex.ru}}
\maketitle              

\begin{abstract}

Variational inequalities play a key role in machine learning research, such as generative adversarial networks, reinforcement learning, adversarial training, and generative models. This paper is devoted to the constrained variational inequality problems with functional constraints (inequality-type constraints). We propose some mirror descent-type algorithms that switch between productive and non-productive steps depending on the values of the functional constraints at iterations, with many different step size rules and stopping criteria. We analyze the proposed algorithms and prove their optimal convergence rate to achieve a solution with desired accuracy, for problems with bounded and monotone operators and Lipschitz convex functional constraints. In addition, we propose a modification of the proposed algorithms by considering each functional constraint in the calculation when we have a productive step, as well as the first constraint that violates the feasibility. This modification can save the running time of algorithms when we have many functional constraints. In addition, we provide an analysis of the proposed algorithms for $\delta$-monotone operators, allowing us to apply the proposed algorithms, as a special case, to constrained minimization problems when we do not have access to the exact information about the subgradient of the objective function. Numerical experiments that illustrate the work and performance of the proposed algorithms are also given.

\keywords{Mirror descent \and Variational Inequality \and Monotone operator \and Bounded operator \and Lipschitz continuous function \and Constrained optimization.}
\end{abstract}

\section*{Introduction}
Variational inequalities (VIs) cover, as a special case, many optimization problems such as minimization problems, saddle point problems, and fixed point problems. They often arise in various mathematical problems, such as optimal control, partial differential equations, mechanics, finance, etc. They play a key role in solving equilibrium and complementarity problems \cite{facchinei2003finite}, in machine learning research such as generative adversarial networks \cite{goodfellow2020generative}, supervised/unsupervised learning \cite{joachims2005support,xu2004maximum}, adversarial training \cite{madry2017towards}, and generative models \cite{daskalakis2017training,gidel2018variational}.

Numerous researchers have dedicated their efforts to exploring theoretical aspects related to the existence and stability of solutions and constructing iterative methods for solving the classical VIs (by classical, we mean the problems without functional ''inequality-type'' constraints). A significant contribution to the development of numerical methods for solving the classical VIs was made in the 1970s, when the extragradient method was proposed in \cite{korpelevich1976extragradient}. More recently, Nemirovski in his seminal work \cite{nemirovski2004prox} proposed a non-Euclidean variant of this method, called the Mirror Prox algorithm, which can be applied to Lipschitz continuous operators.   Different methods with similar complexity were also proposed in \cite{gasnikov2019adaptive,nesterov2007dual}. Besides that, in \cite{nesterov2007dual}, Nesterov proposed a method for variational inequalities with a bounded variation of the operator, i.e., with a non-smooth operator. There is also extensive literature on variations of the extragradient method that avoid taking two steps or two gradient computations per iteration, and so on (see, for example \cite{malitsky2020forward}). 

Another important class of VIs is the problem with functional constraints (inequality-type), see \eqref{main_problem_1} and \eqref{main_problem}. The presence of such constraints makes these problems more difficult to solve. This class of problems arises in many fields of mathematics, among them are economic equilibrium models \cite{Levin1993Mathematical}, constrained Markov potential games \cite{Alatur2023Provably,Jordan2024Independent}, generalized Nash equilibrium problems with jointly-convex constraints \cite{Jordan2023First}, hierarchical programming problems \cite{Migdalas1996Hierarchical}, and in mathematical physics \cite{Baiocchi1988Variational}. See \cite{Antipin2000Solution} for more details and examples. In addition, this class of problems encompasses important applications in machine learning, including reinforcement learning with safety constraints \cite{Tengyu2021}, and learning with fairness constraints \cite{Lowy2022,Muhammad2019}.

For VIs with functional constraints, the previous works have focused on primal-dual algorithms based on the (augmented) Lagrangian function to handle the constraints and penalty methods \cite{Auslender2003Variational,He2004modified,Zhu2003Augmented}. These algorithms and their convergence guarantees crucially depend on information about the optimal Lagrange multipliers. In \cite{Zhang2024Primal},  a primal method was proposed without knowing any information on the optimal Lagrange multipliers, and its convergence rate was proved for the problem with monotone operators under smooth constraints. In \cite{Yang2023Solving}, a first-order method (ACVI) was presented, which combines path-following interior point methods and primal-dual methods. In \cite{Chavdarova2024Primal}, the authors proposed a primal-dual approach to solve the VIs with general functional constraints by taking the last iteration of ACVI. Although there are many works for the VIs with functional constraints, they remain very few compared to the existing works for the classical constrained problem.

In this paper, to solve the variational inequality problem with functional constraints (inequality-type) \eqref{main_problem_1} and \eqref{main_problem}, we propose various mirror descent-type methods with different step size schemes. The mirror descent method, for minimization problems, originated in \cite{Nemirovskii1979efficient,Nemirovsky1983Complexity} and was later analyzed in \cite{Beck2003Mirror}. It is considered as a non-Euclidean extension of standard subgradient methods, which have a long history starting with the method for deterministic unconstrained problems and Euclidean setting in \cite{shor1967} and the generalization for constrained problems in \cite{Polyak1967,demyanov_book}, where the idea of step's switching between the direction of subgradient of the objective and the direction of subgradient of the constraint was suggested. The Mirror Descent method not only generalizes the standard subgradient methods but also achieves a better convergence rate \cite{article:doan_2019}. It is also applicable to optimization problems in Banach spaces where gradient descent is not \cite{article:doan_2019}. Some of the works on first-order methods for convex optimization problems with convex functional constraints include (for example, but not limited to) \cite{Alkousa2024Mirror,article:adaptive_mirror_2018,Stonyakin2018Adaptive,stonyakin2019} for the deterministic setting and \cite{Alkousa2019,Xu2020primal} for the stochastic setting. 

The paper consists of an introduction and five main sections, in addition to four appendices. In Sect. \ref{sect_basics}  we mentioned some basic facts, definitions, and tools for variational inequalities.  Sect. \ref{sectinAlgs} is devoted to the proposed algorithms. We proposed 7 algorithms, with different step size rules and stopping criteria. In Sect. \ref{section_analysis_algs} we analyzed the proposed algorithms and proved their optimal convergence rate for the class of variational inequality problems with bounded monotone operators and Lipschitz functional constraints. In Sect. \ref{sect_modifications}, we propose a modification of the proposed algorithms in the previous section. In this modification, we consider each constraint in the calculation when we have a productive step, and the first constraint that violates the feasibility. In Sect. \ref{sect_numerical} we present numerical experiments that compare the proposed algorithms for some constrained variational inequality and saddle point problems. In the last Section \ref{sec_consl}, we review the obtained results. In the appendices, we provide the proofs of the theorems about the analysis of the proposed algorithms, in addition to their analysis when the operator is $\delta$-monotone. We also provide additional experiments concerning the Forsaken game.

\section{Fundamentals }\label{sect_basics}
Let $(\mathbf{E},\|\cdot\|)$ be a normed finite-dimensional vector space, with an arbitrary norm $\|\cdot\|$, and $\mathbf{E}^*$ be the conjugate space of $\mathbf{E}$ with the following norm
$$
    \|y\|_{*}=\max\limits_{x \in \mathbf{E}}\{\langle y,x\rangle: \|x\|\leq1\},
$$
where $\langle y,x\rangle$ is the value of the continuous linear functional $y \in \mathbf{E}^*$ at $x \in \mathbf{E}$.

Let $Q \subset \mathbf{E}$ be a convex compact set with a diameter $D >0$,  and $\psi: Q \longrightarrow \mathbb{R}$ be a proper closed differentiable and $\sigma$-strongly convex (called prox-function or distance generating function). The corresponding Bregman divergence is defined as 
$$
    V (x, y) = \psi (x) - \psi (y) - \langle \nabla \psi (y), x - y \rangle \quad \forall x , y \in Q. 
$$
For the Bregman divergence, it holds the following inequality
\begin{equation}\label{eq_breg}
V(x, y) \geq \frac{\sigma}{2} \|y - x\|^2 \quad \forall x, y \in Q. 
\end{equation}

For all  $x\in Q$ and $p\in \mathbf{E}^*$, \textit{the proximal mapping operator} is defined as follows
$$
    \operatorname{Mirr}_x (p) = \arg\min\limits_{u\in Q} \big\{ \langle p, u \rangle + V(u, x) \big\}.
$$
We make the simplicity assumption, which means that $\operatorname{Mirr}_x (p)$ is easily computable.

The following well-known lemma describes the main property of the proximal mapping operator for a convex function.

\begin{lemma}\label{fundamental_lemma_deter_case}
Let $f: Q \longrightarrow \mathbb{R}$ be a convex subdifferentiable function over the convex set $Q$ and  $z = \operatorname{Mirr}_{y}(h \nabla f(y))$ for some $h>0$ and $y, z \in Q$. Then for each $x\in Q$ we have
\begin{equation}\label{eq:fundamental_lemma_deter_case}
    h \left(f(y) - f(x)\right) \leq h\langle\nabla f(y), y-x\rangle\leq\frac{h^2}{2}\|\nabla f(y)\|_{*}^2 + V(x,y) - V(x, z).
\end{equation}
\end{lemma}

Consider a set of convex subdifferentiable functionals $g_i: Q \longrightarrow \mathbb{R}$,  $i = 1,2, \ldots, m$. In addition, we assume that all functionals $ g_i $ are Lipschitz-continuous with some constant $ M_{g_i} >0 $, i.e.,
\begin{equation}\label{Lip_condition_for_gm}
    \left|g_i(x)-g_i(y) \right|\leq M_{g_i} \|x-y\| \quad \forall \; x,y \in Q \;\; \text{and} \;\; i =1, \ldots, m.
\end{equation}

It means that at every point $x \in Q$ and for any $ i = 1, \ldots, m$ there is a subgradient $\nabla g_i(x)$, such that $\|\nabla g_i(x)\|_{*} \leq M_{g_i}$.

In this paper, we consider the following constrained variational inequality problem
\begin{equation}\label{main_problem_1}
 \begin{aligned}
    \text{Find} \quad x^* \in Q & : \quad  \langle F(x), x^* - x \rangle \leq 0 \quad \forall x \in Q, 
    \\& \text{and} \quad g_i(x) \leq 0 \quad \forall i = 1, 2 \ldots, m, 
\end{aligned}   
\end{equation}
where $F: Q \longrightarrow \textbf{E}^*$ is a continuous, monotone operator, i.e., 
\begin{equation}\label{eq:CondMonotone}
    \langle F(x) - F(y) , x - y \rangle \geq 0 \quad \forall x, y \in Q.
\end{equation}

It is clear that instead of a set of Lipschitz-continuous functionals  $\{g_i(\cdot)\}_{i=1}^{m}$ we can see one Lipschitz-continuous functional constraint $g: Q \longrightarrow \mathbb{R}$, such that
\begin{equation}\label{functional_of_constraints_g(x)}
    g(x) = \max\limits_{1 \leq i \leq m} \{g_i(x)\}, \quad |g(x)-g(y)|\leq M_g\|x-y\| \; \quad \forall \; x,y\in Q,
\end{equation}
where $M_g = \max_{1 \leq i \leq m} \{M_{g_i}\}$. Thus, the problem \eqref{main_problem_1}, will be equivalent to the following problem
\begin{equation}\label{main_problem}
 \begin{aligned}
    \text{Find} \quad x^* \in Q & : \quad  \langle F(x), x^* - x \rangle \leq 0 \quad \forall x \in Q, 
    \\& \text{and} \quad g(x) \leq 0.
\end{aligned}   
\end{equation}

We say that the operator $F$ is bounded on $Q$, if there exists $L_F >0$ such that 
\begin{equation}\label{bounded_cond}
    \|F (x)\|_* \leq L_F, \quad \forall x \in Q. 
\end{equation}

To emphasize the extensiveness of the problem \eqref{main_problem_1} (or \eqref{main_problem}) without considering the functional constraints (as a special case),  we mention three common special cases for VIs.

\begin{example}[Minimization problem]\label{ex:minproblem}
Let us consider the minimization problem 
\begin{equation}\label{min_problem}
    \min_{x \in Q} f(x), 
\end{equation}
and assume that $F(x) = \nabla f(x)$, where $\nabla f(x)$ denotes the (sub)gradient of $f$ at $x$. Then, if $f$ is convex, it can be proved that $x^* \in Q$  is a solution to \eqref{main_problem_1} (without considering functional constraints) if and only if $x^* \in Q$ is a solution to \eqref{min_problem}.
\end{example}

\begin{example}[Saddle point problem]\label{ex:saddleproblem}  
 Let us consider the saddle point problem
\begin{equation}\label{minmax_problem}
    \min_{u \in Q_u}\max_{v \in Q_v}  f(u, v), 
\end{equation}
\end{example}
and assume that $F(x) : = F(u, v) = \left(\nabla_u f(u,v), -\nabla_v f(u, v)\right)^{\top}$, where $Q = Q_u \times Q_v$ with $Q_u \subseteq \mathbb{R}^{n_u}, Q_v \subseteq \mathbb{R}^{n_v}$. Then if $f$ is convex in $u$ and concave in $v$, it can be proved that $x^* \in Q$ is a solution to \eqref{main_problem_1} (without considering the functional constraints) if and only if $x^* = (u^*, v^*) \in Q$ is a solution to \eqref{minmax_problem}. 

\begin{example}[Fixed point problem]\label{ex:fixedproblem}
Let us consider the fixed point
problem
\begin{equation}\label{fixed_prob}
    \text{Find} \quad x^* \in Q \;\; \text{such that} \quad T(x^*) = x^*,
\end{equation}
where $T: \mathbb{R}^n \longrightarrow \mathbb{R}^n $ is an operator. Taking $F(x)  = x - T(x)$, it can be proved that $x^* \in Q = \mathbb{R}^n$ is a solution to \eqref{main_problem_1} (without considering functional constraints) if $F(x^*) = \textbf{0} \in \mathbb{R}^n$, that is, $x^*$ is a solution to \eqref{fixed_prob}. 
\end{example}

\begin{definition}
For some $\varepsilon >0$, we call a point $\widehat{x} \in Q$ an $\varepsilon$-solution of the problem \eqref{main_problem}, if 
\begin{equation}
    \left\langle F(x), \widehat{x} - x \right \rangle \leq \varepsilon \quad \forall x \in Q, \quad \text{and} \quad g(\widehat{x}) \leq \varepsilon. 
\end{equation}
\end{definition}

\section{Mirror Descent Type Algorithms for VIs with Functional Constraints}\label{sectinAlgs}

In this section, we introduce several mirror descent-based algorithms to solve problem \eqref{main_problem}. We present \textbf{seven} distinct algorithms, differing primarily in their step sizes and stopping criteria. The first two methods, Algorithm \ref{alg_1} and Algorithm \ref{alg_2}, are detailed below. The remaining algorithms (3–7) follow a similar structure and are summarized in Table \ref{algs3-7}.

As can be seen from the items of the proposed algorithms, see e.g., Algorithm \ref{alg_1}, the needed point (output point of the proposed algorithms) is selected among the points $x_i$ for which $g(x_i) \leq \varepsilon$, as an average of these points. Therefore, we will call step $i$ \textit{productive} if $g(x_i) \leq \varepsilon$. If the reverse inequality $g(x_i) > \varepsilon$ holds, then step $i$ will be called \textit{non-productive}.

Let $I$ and $J$ denote the set of indices of productive and non-productive steps, respectively. $|A|$ denotes the cardinality of the set $A$. 

\begin{algorithm}[h!]
\caption{Non-Adaptive Mirror Descent for VIs with functional constraints.}\label{alg_1}
\begin{algorithmic}[1]
\REQUIRE $\varepsilon>0, L_F > 0, M_g >0, D>0$ (the diameter of $Q$), $x_0 \in Q^{\circ}$, $R>0$ such that $\max_{x \in Q}V(x, x_0) \leq R^2$. 
\STATE $I=:\emptyset, J : = \emptyset.$ Set $k = 0.$
\REPEAT
\IF{$g(x_k)\leq\varepsilon$}
\STATE $ h_k^F : = h^F=\frac{\varepsilon}{L_F^2}, \;\; $ $x_{k+1}=\operatorname{Mirr}_{x_k}\left(h^F F(x_k)\right) \;\;$  and add $k$ to $I$,
\ELSE
\STATE $h_k^g : =  h^g=\frac{\varepsilon}{M_g^2}, \;\; $ $x_{k+1}=\operatorname{Mirr}_{x_k}\left(h^g \nabla g(x_k)\right)\;\; $ and add $k$ to $J$,
\ENDIF
\STATE Set $k:= k+1.$
\UNTIL{
\begin{equation}\label{stop1_alg1}
    \text{Stopping criterion 1:} \quad R^2 \leq \frac{\varepsilon^2 |I|}{2L_F^2} + \frac{\varepsilon^2 |J|}{2M_g^2} - \frac{\varepsilon D |J|}{M_g}, 
\end{equation}
or
\begin{equation}\label{stop2_alg1}
    \text{Stopping criterion 2:} \quad R^2 \leq \frac{\varepsilon^2 |I|}{2L_F^2} + \frac{\varepsilon^2 |J|}{2M_g^2}. 
\end{equation}
}
\end{algorithmic}
\end{algorithm}

\begin{algorithm}[h!]
\caption{Adaptive Mirror Descent for VIs with functional constraints.}\label{alg_2}
\begin{algorithmic}[1]
\REQUIRE $\varepsilon>0, M_g >0, D>0$ (the diameter of $Q$), $x_0 \in Q^{\circ}$, $R>0$ such that $\max_{x \in Q}V(x, x_0) \leq R^2$. 
\STATE $I=:\emptyset, J : = \emptyset.$ Set $k = 0.$
\REPEAT
\IF{$g(x_k)\leq\varepsilon$}
\STATE $M_k=\left\|F(x_k)\right\|_*, \;\; h_k^F=\frac{\varepsilon}{M_k^2}, \;\; $ $x_{k+1}=\operatorname{Mirr}_{x_k}\left(h_k^F F(x_k)\right)\;\;$  and add $k$ to $I$,
\ELSE
\STATE $M_k=\left\|\nabla g(x_k)\right\|_*, \;\; h_k^g=\frac{\varepsilon}{M_k^2},\;\;$  $x_{k+1}=\operatorname{Mirr}_{x_k}\left(h_k^g \nabla g(x_k)\right)\;\; $ and add $k$ to $J$,
\ENDIF
\STATE Set $k:= k+1.$
\UNTIL{
\begin{equation}\label{stop1_alg2}
\text{Stopping criterion 1:} \quad R^2 \leq \frac{\varepsilon^2}{2} \sum_{i = 0}^{k - 1} \frac{1}{M_i^2} - M_g D\varepsilon \sum_{i \in J} \frac{1}{\|\nabla g(x_i)\|_*^2}, 
\end{equation}
or
\begin{equation}\label{stop2_alg2}
\text{Stopping criterion 2:} \quad R^2 \leq \frac{\varepsilon^2}{2} \sum_{i = 0}^{k - 1} \frac{1}{M_i^2}.  
\end{equation}
}
\end{algorithmic}
\end{algorithm}

\begin{table}[h!]
\centering
\caption{The proposed Algorithms 3 -- 7. Mirror Descent Algorithms for VIs with functional constraints.}
\scalebox{0.67}{
{\begin{tabular}{|c|c|c|c|c|c|}
\hline
& $g(x_k) \leq$ & $h_k^F =$ & $h_k^g =$ & Stopping criterion 1 & Stopping criterion 2  \\ \hline
\textbf{Algorithm 3} & $\varepsilon M_g$ & $\frac{\varepsilon}{\|F(x_k)\|_*^2}$ & $\frac{\varepsilon}{M_g}$ & $R^2 \leq \frac{\varepsilon^2}{2} \sum\limits_{i \in I} \frac{1}{\|F(x_i)\|_*^2}  + \frac{\varepsilon^2}{2} |J| - \varepsilon D |J|$  & $R^2 \leq \frac{\varepsilon^2}{2}\sum\limits_{i \in I} \frac{1}{\|F(x_i)\|_*^2} + \frac{\varepsilon^2}{2} |J|$  \\ \hline
\textbf{Algorithm 4} & $\varepsilon$ & $\frac{\varepsilon}{\|F(x_k)\|_*}$ & $\frac{\varepsilon}{\left\|\nabla g(x_k)\right\|_*^2}$ & $R^2 \leq \frac{\varepsilon^2}{2} |I| + \left(\frac{\varepsilon^2}{2}  - \varepsilon M_g D\right) \sum\limits_{i \in J} \frac{1}{\|\nabla g(x_i)\|_*^2}$ & $R^2 \leq \frac{\varepsilon^2}{2} \left(|I| 
+ \sum\limits_{i \in J} \frac{1}{\|\nabla g(x_i)\|_*^2} \right)$ \\ \hline
\textbf{Algorithm 5} & $\varepsilon M_g$ & $\frac{\varepsilon}{\left\|F(x_k)\right\|_*}$ & $\frac{\varepsilon}{M_g}$ & $R^2 \leq \frac{\varepsilon^2}{2} \left(|I| + |J|\right) -\varepsilon D|J|$ & $R^2 \leq \frac{\varepsilon^2}{2} \left(|I| + |J|\right)$  \\ \hline
\textbf{Algorithm 6} & $\varepsilon$ & $\frac{\varepsilon}{M_g \left\|F(x_k)\right\|_* }$ & $\frac{\varepsilon}{M_g^2}$ & $R^2 \leq \frac{\varepsilon^2}{2M_g^2} \left(|I| + |J|\right) - \frac{\varepsilon D|J|}{M_g}$ & $R^2 \leq \frac{\varepsilon^2}{2M_g^2} \left(|I| + |J|\right)$ \\ \hline
\textbf{Algorithm 7} & $\varepsilon$ & \shortstack{$\theta \left(\sum\limits_{t = 0}^{k} M_t^2\right)^{-1/2}$, \\ $M_t =\|F(x_t)\|_*$}   & \shortstack{$\theta \left(\sum\limits_{t = 0}^{k} M_t^2\right)^{-1/2}$, \\ $M_t =\|\nabla g(x_t)\|_*$} & $k \geq  \frac{2 \theta}{\varepsilon} \left(\sum\limits_{t = 0}^{k} M_t^2\right)^{-1/2} + \frac{|J| M_g D}{\varepsilon}$ & $k \geq \frac{2 \theta}{\varepsilon} \left(\sum\limits_{t = 0}^{k} M_t^2\right)^{-1/2}$  \\ \hline
\end{tabular}}
}
\label{algs3-7}
\end{table}

\section{Analysis of the proposed algorithms}\label{section_analysis_algs}

In this section, we provide an analysis of the proposed Algorithms 1 --- 7. 

Assume that $F: Q \longrightarrow \mathbf{E}^*$ is a continuous, monotone, and bounded operator (see \eqref{bounded_cond}), $g (x): = \max\limits_{1 \leq i \leq m } \{g_i(x)\}$ is an $M_g$-Lipschitz convex function, where $g_i : Q \longrightarrow \mathbb{R}$ for $i = 1, \ldots, m$ are $M_{g_i}$-Lipschitz,  and $M_g=  \max\limits_{1 \leq i \leq m } \left\{M_{g_i}\right\}$.

\subsection{Analysis of Algorithm \ref{alg_1}}

For Algorithm \ref{alg_1}, we have the following result. 
\begin{theorem}\label{main_theorem_alg1}
By Algorithm \ref{alg_1}, we get a point $\widehat{x} = \frac{1}{\sum_{i \in I} h_i^F} \sum_{i \in I} h_i^F x_i$, such that 
\begin{equation}\label{feasibility_g_alg1}
    g(\widehat{x}) \leq \varepsilon,
\end{equation}
and 
\begin{enumerate}
\item with stopping criterion 1, we get 
\begin{equation}\label{alg1stop1_quality_F}
    \left\langle F(x), \widehat{x} - x \right\rangle < \varepsilon \quad \forall x \in Q, 
\end{equation}

\item with stopping criterion 2, we get 
\begin{equation}\label{alg1stop2_quality_F}
    \left\langle F(x), \widehat{x} - x \right\rangle < \varepsilon + \frac{D L_F^2 |J|}{M_g |I|}, \quad \forall x \in Q.
\end{equation}
\end{enumerate}
\end{theorem}
\begin{proof}
The proof is in the Appendix. See Subsec. \ref{append_proof_theo_alg1}. 
\end{proof}

\begin{remark}
Since $M_g \leq \max \left\{ M_g, L_F \right\}$ and $L_F \leq \max \left\{ M_g, L_F \right\}$, then from the stopping criterion 2 of Algorithm \ref{alg_1}, we find
$$
    \frac{\varepsilon^2}{2 L_F^2} |I| + \frac{\varepsilon^2}{2 M_g^2} |J| \geq \frac{\varepsilon^2 (|I| + |J|)}{2 \max\left\{L_F^2, M_g^2\right\}} = \frac{\varepsilon^2 k }{2 \max\left\{L_F^2, M_g^2\right\}} \geq R^2. 
$$
Thus, after no more than $k = \left\lceil \frac{2 R^2 \max\left\{L_F^2, M_g^2\right\}}{\varepsilon^2}  \right\rceil = \mathcal{O}\left(\frac{1}{\varepsilon^2}\right)$ iterations of Algorithm \ref{alg_1}, with the second stooping criterion, we get \eqref{feasibility_g_alg1} and \eqref{alg1stop2_quality_F}. 
\end{remark}

\subsection{Analysis of Algorithm \ref{alg_2}}

For Algorithm \ref{alg_2}, we have the following result. 
\begin{theorem}\label{main_theorem_alg2}
By Algorithm \ref{alg_2}, we get a point $\widehat{x}= \frac{1}{\sum_{i \in I} h_i^F} \sum_{i \in I} h_i^F x_i $, such that, 
\begin{equation}\label{feasibility_g_alg2}
    g(\widehat{x}) \leq \varepsilon,
\end{equation}
and 
\begin{enumerate}
\item with stopping criterion 1, we get 
\begin{equation}\label{alg2stop1_quality_F}
    \left\langle F(x), \widehat{x} - x \right\rangle < \varepsilon \quad \forall x \in Q,
\end{equation}

\item with stopping criterion 2, we get 
\begin{equation}\label{alg2stop2_quality_F}
    \left\langle F(x), \widehat{x} - x \right\rangle < \varepsilon +   M_g D\sum_{i \in J} \frac{1}{\|\nabla g(x_i)\|_*^2} \left( \sum_{i \in I} \frac{1}{\|F(x_i)\|_*^2}  \right)^{-1} \quad \forall x \in Q. 
\end{equation}
\end{enumerate}
\end{theorem}
\begin{proof}
The proof is in the Appendix. See Subsec. \ref{append_proof_theo_alg2}. 
\end{proof}

\begin{remark}
Since $M_i \leq \max \left\{ M_g, L_F \right\}$ for every $i = 0, 1, \ldots, k - 1$, then from the stopping criterion 2 of Algorithm \ref{alg_2}, we find
$$
    \frac{\varepsilon^2}{2} \sum_{i = 0}^{k -1}  \frac{1}{M_i^2}\geq \frac{\varepsilon^2}{2} \sum_{i = 0}^{k -1}  \frac{1}{\max \left\{ M_g^2, L_F^2 \right\}} = \frac{\varepsilon^2 k}{2 \max \left\{ M_g^2, L_F^2 \right\}} \geq R^2. 
$$
Thus, after no more than $k = \left\lceil \frac{2 R^2 \max\left\{L_F^2, M_g^2\right\}}{\varepsilon^2}  \right\rceil = \mathcal{O}\left(\frac{1}{\varepsilon^2}\right)$ iterations of Algorithm \ref{alg_2}, with the second stooping criterion, we get \eqref{feasibility_g_alg2} \eqref{alg2stop2_quality_F}. 
\end{remark}

\subsection{Analysis of Algorithm 3}

For Algorithm 3, we have the following result. 
\begin{theorem}\label{main_theorem_alg3}
By Algorithm 3, we get a point $\widehat{x}= \frac{1}{\sum_{i \in I} h_i^F} \sum_{i \in I} h_i^F x_i $, such that 
\begin{equation}\label{feasibility_g_alg3}
    g(\widehat{x}) \leq \varepsilon M_g,
\end{equation}
and 
\begin{enumerate}
\item with stopping criterion 1, we get 
\begin{equation}\label{alg3stop1_quality_F}
    \left\langle F(x), \widehat{x} - x \right\rangle < \varepsilon \quad \forall x \in Q,
\end{equation}

\item with stopping criterion 2, we get 
\begin{equation}\label{alg3stop2_quality_F}
    \left\langle F(x), \widehat{x} - x \right\rangle < \varepsilon + D|J| \left(\sum_{i \in I} \frac{1}{\|F(x_i)\|_*^2}\right)^{-1} \quad \forall x \in Q. 
\end{equation}
\end{enumerate}
\end{theorem}
\begin{proof}
The proof is in the Appendix. See Subsec. \ref{append_proof_theo_alg3}. 
\end{proof}

\begin{remark}
Since $\|F(x_i)\|_* \leq L_F$ for every $i \in I$, then from the stopping criterion 2 of Algorithm 3, we find
$$
    |J| + \sum_{i \in I} \frac{1}{\|F(x_i)\|_*^2} \geq |J| + \frac{|I|}{L_F^2} \geq \frac{|I| + |J|}{\max\{1 , L_F^2\}}. 
$$
Thus, after no more than $k = \left\lceil \frac{2 R^2 \max\left\{1, L_F^2\right\}}{\varepsilon^2}  \right\rceil = \mathcal{O}\left(\frac{1}{\varepsilon^2}\right)$ iterations of Algorithm 3, with the second stooping criterion, we have \eqref{feasibility_g_alg3} and \eqref{alg3stop2_quality_F}. 
\end{remark}

\subsection{Analysis of Algorithm 4}

For Algorithm 4, we have the following result. 

\begin{theorem}\label{main_theorem_alg4}
By Algorithm 4, we get a point $\widehat{x}= \frac{1}{\sum_{i \in I} h_i^F} \sum_{i \in I} h_i^F x_i $, such that \begin{equation}\label{feasibility_g_alg4}
    g(\widehat{x}) \leq \varepsilon,
\end{equation}
and 
\begin{enumerate}
\item with stopping criterion 1, we get 
\begin{equation}\label{alg4stop1_quality_F}
    \left\langle F(x), \widehat{x} - x \right\rangle < \varepsilon L_F \quad \forall x \in Q,
\end{equation}

\item with stopping criterion 2, we get 
\begin{equation}\label{alg4stop2_quality_F}
    \left\langle F(x), \widehat{x} - x \right\rangle < \varepsilon L_F + \frac{M_g D L_F}{|I|} \sum_{i \in I} \frac{1}{\|\nabla g(x_i)\|_*^2} \quad \forall x \in Q . 
\end{equation}
\end{enumerate}
\end{theorem}
\begin{proof}
The proof is in the Appendix. See Subsec. \ref{append_proof_theo_alg4}. 
\end{proof}

\begin{remark}
Since $\|\nabla g(x_i)\|_* \leq M_g$ for every $i \in J$, then from the stopping criterion 2 of Algorithm 4, we find
$$
    |I| + \sum_{i \in J} \frac{1}{\|\nabla g(x_i)\|_*^2} \geq |I| + \frac{|J|}{M_g^2} \geq \frac{|I| + |J|}{\max\{1 ,M_g^2\}}. 
$$
Thus, after no more than $k = \left\lceil \frac{2 R^2 \max\left\{1, M_g^2\right\}}{\varepsilon^2}  \right\rceil = \mathcal{O}\left(\frac{1}{\varepsilon^2}\right)$ iterations of Algorithm 4, with the second stooping criterion, we have \eqref{feasibility_g_alg4} and \eqref{alg4stop2_quality_F}. 
\end{remark}

\subsection{Analysis of Algorithm 5}

For Algorithm 5, we have the following result. 
\begin{theorem}\label{main_theorem_alg5}
By Algorithm 5, we get a point $\widehat{x}= \frac{1}{\sum_{i \in I} h_i^F} \sum_{i \in I} h_i^F x_i $, such that 
\begin{equation}\label{feasibility_g_alg5}
    g(\widehat{x}) \leq \varepsilon M_g,
\end{equation}
and 
\begin{enumerate}
\item with stopping criterion 1, we get 
\begin{equation}\label{alg5stop1_quality_F}
    \left\langle F(x), \widehat{x} - x \right\rangle < \varepsilon L_F \quad \forall x \in Q,
\end{equation}

\item with stopping criterion 2, we get 
\begin{equation}\label{alg5stop2_quality_F}
    \left\langle F(x), \widehat{x} - x \right\rangle < \varepsilon L_F + \frac{ D L_F |J|}{|I|} \quad \forall x \in Q. 
\end{equation}
\end{enumerate}
\end{theorem}
\begin{proof}
The proof is in the Appendix. See Subsec. \ref{append_proof_theo_alg5}. 
\end{proof}

\begin{remark}
From the stopping criterion 2 of Algorithm 5, we find
that after no more than $k = \left\lceil \frac{2 R^2}{\varepsilon^2}  \right\rceil = \mathcal{O}\left(\frac{1}{\varepsilon^2}\right)$ iterations of Algorithm 5, we have \eqref{feasibility_g_alg5} and \eqref{alg5stop2_quality_F}. 
\end{remark}

\subsection{Analysis of Algorithm 6}

For Algorithm 6, we have the following result. 
\begin{theorem}\label{main_theorem_alg6}
By Algorithm 6, we get a point $\widehat{x}= \frac{1}{\sum_{i \in I} h_i^F} \sum_{i \in I} h_i^F x_i $, such that \begin{equation}\label{feasibility_g_alg6}
    g(\widehat{x}) \leq \varepsilon,
\end{equation}
and 
\begin{enumerate}
\item with stopping criterion 1, we get 
\begin{equation}\label{alg6stop1_quality_F}
    \left\langle F(x), \widehat{x} - x \right\rangle < \frac{\varepsilon L_F}{M_g} \quad \forall x \in Q,
\end{equation}

\item with stopping criterion 2, we get 
\begin{equation}\label{alg6stop2_quality_F}
    \left\langle F(x), \widehat{x} - x \right\rangle < \frac{\varepsilon L_F}{M_g} + \frac{D L_F |J|}{M_g |I|} \quad \forall x \in Q. 
\end{equation}
\end{enumerate}
\end{theorem}
\begin{proof}
    The proof is in the Appendix. See Subsec. \ref{append_proof_theo_alg6}. 
\end{proof}

\begin{remark}
From the stopping criterion 2 of Algorithm 6, we find
that after no more than $k = \left\lceil \frac{2 R^2 M_g^2}{\varepsilon^2}  \right\rceil = \mathcal{O}\left(\frac{1}{\varepsilon^2}\right)$ iterations of Algorithm 6, we have \eqref{feasibility_g_alg6} and \eqref{alg6stop2_quality_F}. 
\end{remark}

\subsection{Analysis of Algorithm 7}

In Algorithm 7, we mention that the parameter $\theta >0$ is such that $\max\limits_{x, y \in Q}V(x, y) \leq \theta^2$. For Algorithm 7, we have the following result. 
\begin{theorem}\label{main_theorem_alg7}
By Algorithm 7, we get a point $\widehat{x}= \frac{1}{|I|} \sum_{i \in I}x_i $, such that \begin{equation}\label{feasibility_g_alg7}
    g(\widehat{x}) \leq \varepsilon,
\end{equation}
and 
\begin{enumerate}
\item with stopping criterion 1, we get 
\begin{equation}\label{alg7stop1_quality_F}
    \left\langle F(x), \widehat{x} - x \right\rangle < \varepsilon \quad \forall x \in Q,
\end{equation}

\item with stopping criterion 2, we get 
\begin{equation}\label{alg7stop2_quality_F}
    \left\langle F(x), \widehat{x} - x \right\rangle < \varepsilon + \frac{|J| M_g D}{|I|} \quad \forall x \in Q. 
\end{equation}
\end{enumerate}
\end{theorem}
\begin{proof}
    The proof is in the Appendix. See Subsec. \ref{append_proof_theo_alg7}. 
\end{proof}

\begin{remark}
From the stopping criterion 2 of Algorithm 7, we find
that after no more than $k = \left\lceil \frac{4 \theta \max \left\{L_F^2, M_g^2\right\}}{\varepsilon^2}  \right\rceil = \mathcal{O}\left(\frac{1}{\varepsilon^2}\right)$ iterations of Algorithm 7, we have \eqref{feasibility_g_alg7} and \eqref{alg7stop2_quality_F}. 
\end{remark}

\section{Modification of the Algorithms \ref{alg_1}---7 for VIs with many functional constraints}\label{sect_modifications}

In this section, we propose a modification of the previously proposed Algorithms \ref{alg_1}---7, to solve the problem \eqref{main_problem_1}. The modification was first proposed for the minimization problems in \cite{Stonyakin2018Adaptive}. Instead of transforming all functional constraints into one that equals the maximum of them as in problem \eqref{main_problem}, we consider each constraint $g_i$ in the calculation when we have a productive step, and the first constraint that violates the feasibility. Thus, this proposed modification saves the algorithms' running time by considering not all functional constraints on non-productive steps. 

In the following, we list the modifications of Algorithm \ref{alg_2}, where the same will be for the remaining Algorithms \ref{alg_1} and 3---7. 

The results of Theorems \ref{main_theorem_alg1}--\ref{main_theorem_alg7} extend to all modified versions of Algorithms \ref{alg_1}--7. For brevity, we present the proof only for the modification of Algorithm \ref{alg_2}, as the arguments for the remaining algorithms are analogous.


\setcounter{algorithm}{7}
\begin{algorithm}[H]
\caption{Modification of Algorithm \ref{alg_2} with many functional constraints.}\label{alg_2_modif}
\begin{algorithmic}[1]
\REQUIRE $\varepsilon>0, M_g = \max_{1 \leq i \leq m} \{M_{g_i}\}>0, D>0$ (the diameter of $Q$), $x_0 \in Q^{\circ}$, $R>0$ such that $\max_{x \in Q}V(x, x_0) \leq R^2$. 
\STATE $I=:\emptyset, J : = \emptyset.$
\STATE Set $k = 0.$
\REPEAT
\IF{{\color{black}$g_i(x_k)\leq \varepsilon \; \forall i = 1,2, \ldots, m$}}
\STATE $M_k=\left\|F(x_k)\right\|_*, \;\; h_k^F=\frac{\varepsilon}{M_k^2},\;\;$ \STATE $x_{k+1}=\operatorname{Mirr}_{x_k}\left(h_k^F F(x_k)\right)\;\; $ and add $k$ to $I$,
\ELSE
\STATE (i.e., $\exists N = N(k) \in \{1, 2, \ldots, m\}$, s.t., $g_{N(k)} > \varepsilon$)
\STATE $M_k=\left\|\nabla g_{N(k)}(x_k)\right\|_*, \;\; h_k^{g_{N(k)}}=\frac{\varepsilon}{M_k^2},$
\STATE ${\color{black}x_{k+1}=\operatorname{Mirr}_{x_k}\left(h_k^{g_{N(k)}} \nabla g_{N(k)}(x_k)\right)}\;\;$ and add $k$ to $J$,
\ENDIF
\STATE Set $k:= k+1.$
\UNTIL{
\begin{equation}\label{stop1_alg2_mod}
    \text{Stopping criterion 1:} \quad R^2 \leq \frac{\varepsilon^2}{2} \sum_{i = 0}^{k - 1} \frac{1}{M_i^2} - M_g D\varepsilon \sum_{i \in J} \frac{1}{\|\nabla {\color{black}g_{N(i)}}(x_i)\|_*^2}, 
\end{equation}
or
\begin{equation}\label{stop2_alg2_mod}
    \text{Stopping criterion 2:} \quad R^2 \leq \frac{\varepsilon^2}{2} \sum_{i = 0}^{k - 1} \frac{1}{M_i^2}.  
\end{equation}
}
\end{algorithmic}
\end{algorithm}

For Algorithm \ref{alg_2_modif}, we have the following result. 
\begin{theorem}\label{main_theorem_alg2_modif}
By Algorithm \ref{alg_2_modif}, we get a point $\widehat{x}= \frac{1}{\sum_{i \in I} h_i^F} \sum_{i \in I} h_i^F x_i $, such that, 
\begin{equation*}\label{feasibility_g_alg2_mod}
    g_i(\widehat{x}) \leq \varepsilon, \quad \forall i \in \{1,2,\ldots, m\},
\end{equation*}
and 
\begin{enumerate}
\item with stopping criterion 1 \eqref{stop1_alg2_mod}, we get 
\begin{equation*}\label{alg2_modstop1_quality_F}
    \left\langle F(x), \widehat{x} - x \right\rangle < \varepsilon \quad \forall x \in Q,
\end{equation*}

\item with stopping criterion 2 \eqref{stop2_alg2_mod},
\begin{equation*}
    \left\langle F(x), \widehat{x} - x \right\rangle < \varepsilon +   M_g D\sum_{i \in J} \frac{1}{\|\nabla g_{N(i)}(x_i)\|_*^2} \left( \sum_{i \in I} \frac{1}{\|F(x_i)\|_*^2}  \right)^{-1} \; \forall x \in Q. 
\end{equation*}
\end{enumerate}
\end{theorem}
\begin{proof}
The proof is in the Appendix. See Subsec. \ref{append_proof_theo_modalg2}. 
\end{proof}

\section{Numerical experiments}\label{sect_numerical}

In this section, we mention the results of the conducted experiments for the proposed algorithms in Sect. \ref{sectinAlgs}, for one example of the variational inequality problem (HpHard problem) with functional constraints. 

All experiments were implemented in Python 3.4, on a computer fitted with Intel(R) Core(TM) i7-8550U CPU @ 1.80GHz, 1992 Mhz, 4 Core(s), 8 Logical Processor(s). The RAM of the computer is 8 GB.

Let us consider the problem \eqref{main_problem}, with the following functional constraints, 
\begin{equation}
    g(x) = \max_{1 \leq i \leq m} \left\{g_i(x) = \langle a_i, x \rangle - b_i \right\},
\end{equation}
where $a_i \in \mathbb{R}^n$ and $b_i \in \mathbb{R}$ for any $i \in \{1, \ldots, m\}$. 

The vectors $a_j$ and constants $b_j$, for $ j = 1, 2, \ldots, m$,  are randomly generated from a uniform distribution over $[0, 1)$. In our experiments, we take $m = 10$.

\begin{example}\label{example_2}
In this example, we consider the HpHard (or Harker-Pang) problem \cite{Qu2024extra}. Let $F: \mathbb{R}^n \longrightarrow \mathbb{R}^n$ be an operator defined by
\begin{equation}
    F(x) = Kx + q, \quad K = A A^{\top } + B + C,  q \in \mathbb{R}^n,
\end{equation}
where $A \in \mathbb{R}^{n \times n}$ is a matrix, $B \in \mathbb{R}^{n \times n}$ is a skew-symmetric matrix and $C \in \mathbb{R}^{n \times n}$ is a diagonal matrix with non-negative diagonal entries. Therefore, it follows that $K$ is positive semidefinite. The operator $F$ is bounded in the unit ball and monotone, where $L_F = \|K\|_2 + \|q\|_2$. For $q = \textbf{0} \in \mathbb{R}^n$, the solution of the problem \eqref{main_problem}, is $x^* = \textbf{0} \in \mathbb{R}^n$. 
\end{example}

\begin{figure}[htp]
\centering
\begin{subfigure}{0.49\linewidth}
\includegraphics[width=\linewidth]{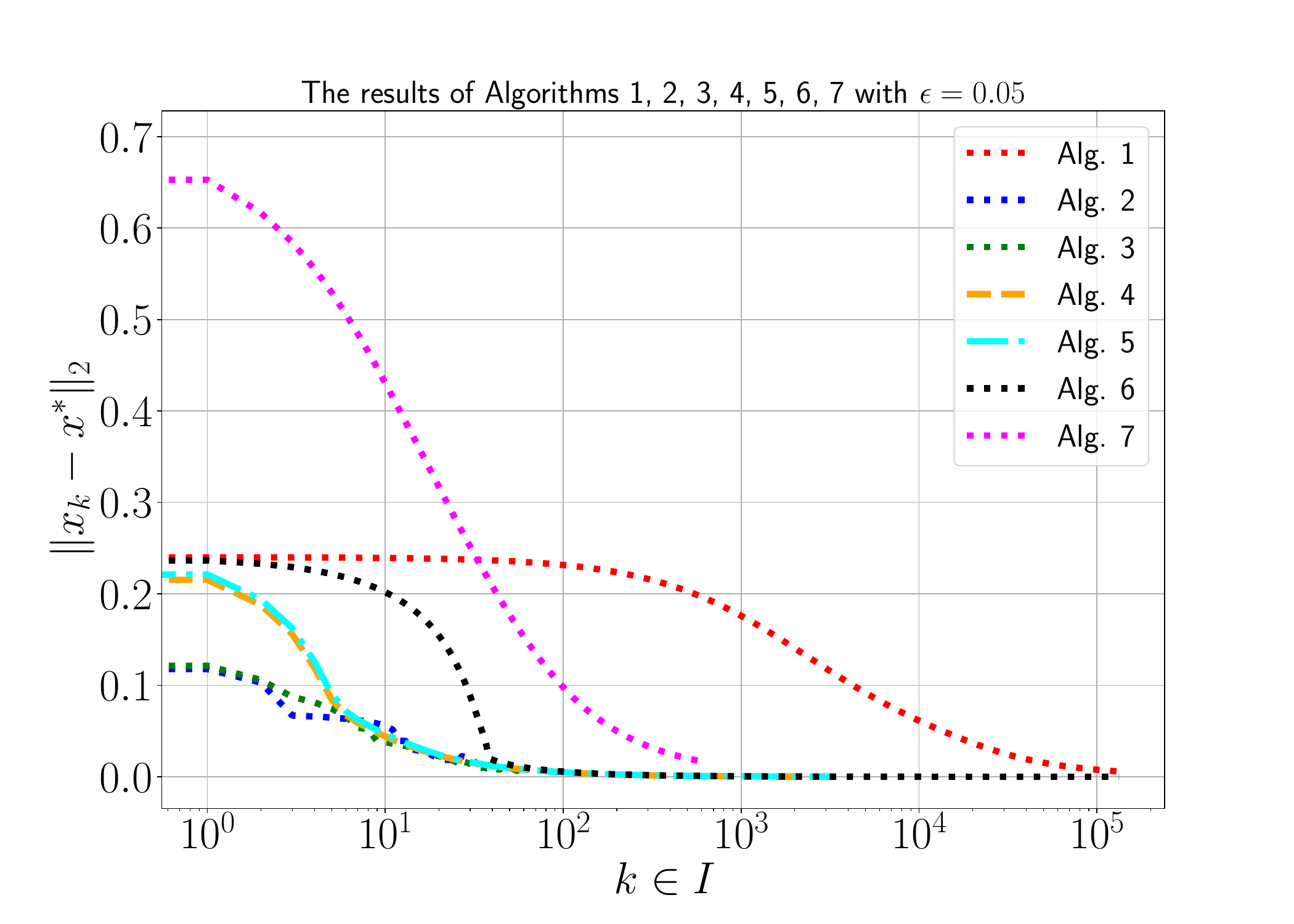}
\end{subfigure}
\begin{subfigure}{0.49\linewidth}
\includegraphics[width=\linewidth]{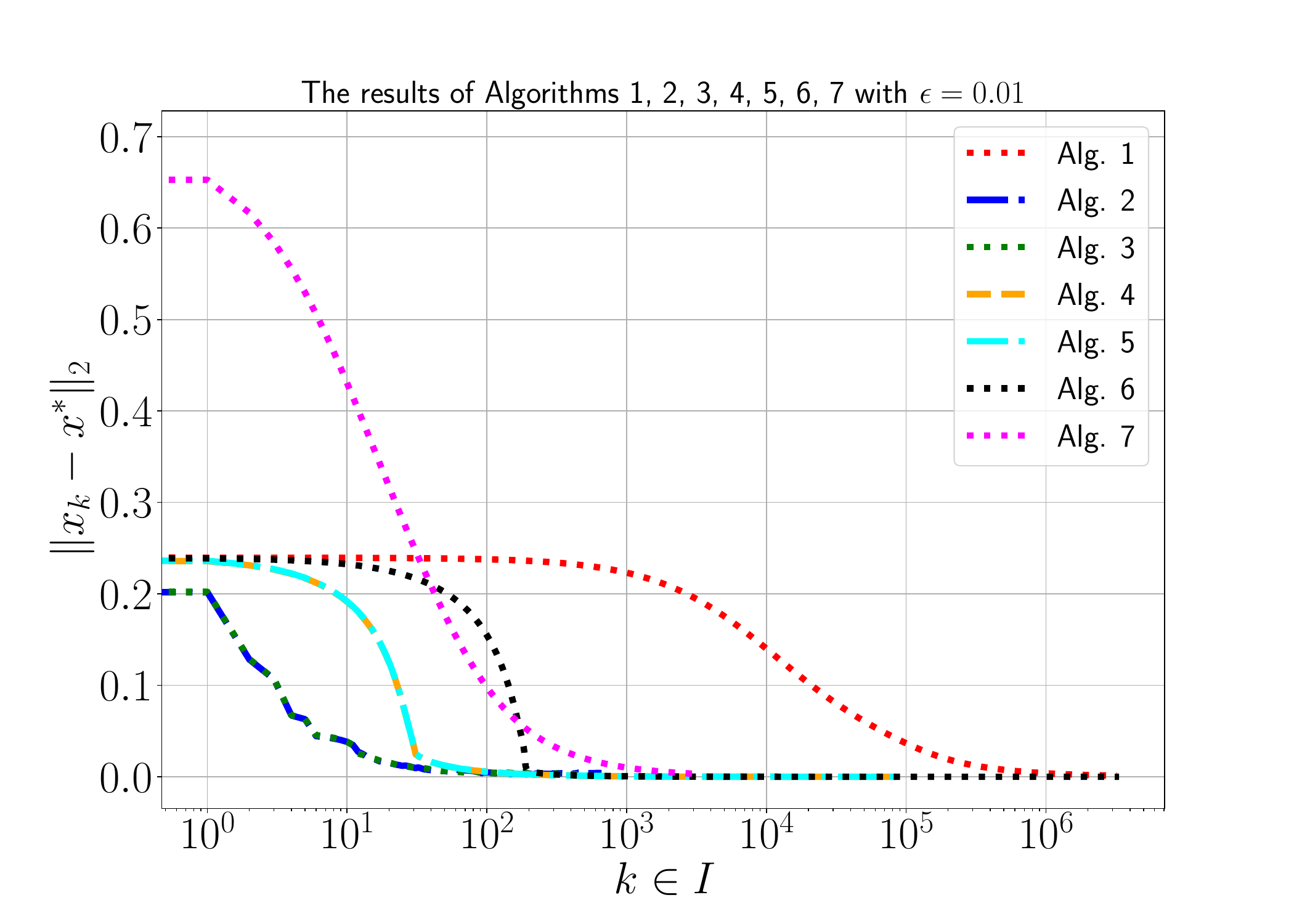}
\end{subfigure}
\caption{ The results of Algorithms \ref{alg_1}---7 (with first stopping criterion), for Example \ref{example_2} with $n = 100, m = 10$,   $ \varepsilon = 0.05$ (left), and $\varepsilon = 0.01$  (right).}
\label{figs:VI2Alg1_7}
\end{figure}

\begin{table}[htp]
\begin{center}
\scalebox{0.78}{
\begin{tabular}{|c|ccccccc|}
\hline
\multirow{2}{*}{} & \multicolumn{1}{c|}{Alg. \ref{alg_1}} & \multicolumn{1}{c|}{Alg. \ref{alg_2}} & \multicolumn{1}{c|}{Alg. 3} & \multicolumn{1}{c|}{Alg. 4} & \multicolumn{1}{c|}{Alg. 5} & \multicolumn{1}{c|}{Alg. 6} & Alg. 7 \\ \cline{2-8} 
& \multicolumn{7}{c|}{$\varepsilon = 0.05$}  \\ \hline
Iters. & \multicolumn{1}{c|}{129005} & \multicolumn{1}{c|}{\textbf{161}} & \multicolumn{1}{c|}{\textbf{80}} & \multicolumn{1}{c|}{3542} & \multicolumn{1}{c|}{3360} & \multicolumn{1}{c|}{133169} &  \textbf{604} \\ \hline
Time& \multicolumn{1}{c|}{18.717003} & \multicolumn{1}{c|}{\textbf{0.021646}} & \multicolumn{1}{c|}{\textbf{0.016614}} & \multicolumn{1}{c|}{0.748257} & \multicolumn{1}{c|}{0.517956} & \multicolumn{1}{c|}{37.257767} & \textbf{0.161242} \\ \hline
Estim. & \multicolumn{1}{c|}{0.05} & \multicolumn{1}{c|}{\textbf{0.05}} & \multicolumn{1}{c|}{\textbf{0.05}} & \multicolumn{1}{c|}{0.306} & \multicolumn{1}{c|}{0.306} & \multicolumn{1}{c|}{0.0492} & \textbf{0.05} \\ \hline
& \multicolumn{7}{c|}{$\varepsilon = 0.01$}  \\ \hline
 Iters.& \multicolumn{1}{c|}{3232248} & \multicolumn{1}{c|}{\textbf{2398}} & \multicolumn{1}{c|}{\textbf{712}} & \multicolumn{1}{c|}{86713} & \multicolumn{1}{c|}{85600} & \multicolumn{1}{c|}{3336676} & \textbf{3020} \\ \hline
Time& \multicolumn{1}{c|}{563.808768} & \multicolumn{1}{c|}{\textbf{0.583022}} & \multicolumn{1}{c|}{\textbf{0.171868}} & \multicolumn{1}{c|}{15.274862} & \multicolumn{1}{c|}{14.675128} & \multicolumn{1}{c|}{758.070035} &  \textbf{0.59175} \\ \hline
Estim.& \multicolumn{1}{c|}{0.01} & \multicolumn{1}{c|}{\textbf{0.01}} & \multicolumn{1}{c|}{\textbf{0.01}} & \multicolumn{1}{c|}{0.6125} & \multicolumn{1}{c|}{0.6125} & \multicolumn{1}{c|}{0.0098} & \textbf{0.01 }\\ \hline
\end{tabular} }
\caption{The results of Algorithms \ref{alg_1}---7 (with first stopping criterion), for Example \ref{example_2} with $L_F \approx 6.125326, M_g \approx 6.22351, n = 100, m = 10$, $ \varepsilon = 0.05, 0.01$. In this table, Iters. denotes the number of iterations of each algorithm, Time denotes the running time of each algorithm in seconds, and Estim. denotes the theoretical estimate of an achieved solution by each algorithm.  }
\label{Table_VI_2}  
\end{center}
\end{table}

\section{Conclusion}\label{sec_consl}

In this paper, we studied the variational inequality problem with inequality functional constraints. To solve this problem, we proposed seven mirror descent-type algorithms and a modification of them with different step size rules and stopping criteria. We provided the analysis of the convergence of all algorithms for the class of problems with bounded and $\delta$-monotone (the monotonicity case corresponds to $\delta = 0$) operators and Lipschitz continuous functional (inequality type) constraints. In the proposed modification, we consider each functional constraint in the calculation when we have a productive step (feasible point), and the first constraint that violates the feasibility.  We proved the optimal convergence rate of the proposed algorithms for the class of problems under consideration. The conducted numerical experiments compare the work of the proposed algorithms for a variational inequality problem (HpHard or Harker-Pang) with functional constraints and for the forsaken game (which is a min-max problem).

\newpage 

\section{Appendix 1: The Proofs of Theorems in Section \ref{section_analysis_algs}}

\subsection{Proof Theorem \ref{main_theorem_alg1}}\label{append_proof_theo_alg1}

\begin{proof}
For any $i \in I$, and $x \in Q$, we have 
\begin{align} \label{alg1_eq1_I}
    h^F \left\langle F(x_i), x_i - x\right \rangle & \leq \frac{(h^F)^2}{2} \|F(x_i)\|_*^2 + V(x, x_i) - V(x, x_{i+1}) \nonumber
    \\& 
    \leq \frac{\varepsilon^2}{2 L_F^2} + V(x, x_i) - V(x, x_{i+1}). 
\end{align}
For any $i \in J$, and $x \in Q$, we have 
\begin{align} \label{alg1_eq1_J}
    h^g \left( g(x_i)  - g(x)\right) & \leq \frac{(h^g)^2}{2} \|\nabla g(x_i)\|_*^2 + V(x, x_i) - V(x, x_{i+1}) \nonumber
    \\& 
    \leq \frac{\varepsilon^2}{2 M_g^2} + V(x, x_i) - V(x, x_{i+1}) .
\end{align}

Summing up inequalities \eqref{alg1_eq1_I} and \eqref{alg1_eq1_J}, from $i = 0$ to $i = k-1$ for any $k \geq 1$, we get the following 
\begin{align*}\label{alg1_eq0}
    h^F  \sum_{i \in I} \left\langle F(x_i), x_i - x\right \rangle + h^g \sum_{i \in J} \left( g(x_i)  - g(x)\right)  & \leq \frac{\varepsilon^2}{2 L_F^2} |I| + \frac{\varepsilon^2}{2 M_g^2} |J|
    \\& \quad + \sum_{i = 0}^{k-1} \left(V(x, x_i) - V(x, x_{i+1})\right) 
    \\& \leq \frac{\varepsilon^2}{2 L_F^2} |I| + \frac{\varepsilon^2}{2 M_g^2} |J| + R^2,
\end{align*}
In the last inequality, we used the fact
\begin{equation*}
    \sum_{i = 0}^{k-1} \left(V(x, x_i) - V(x, x_{i+1})\right) = V(x, x_0) - V(x, x_k) \leq V(x, x_0) \leq R^2. 
\end{equation*}

Thus, we have 
\begin{equation} \label{alg1_eq1}
     h^F \sum_{i \in I}  \left\langle F(x_i), x_i - x\right \rangle \leq \frac{\varepsilon^2}{2 L_F^2} |I| + \frac{\varepsilon^2}{2 M_g^2} |J| + R^2 - h^g  \sum_{i \in J}  \left( g(x_i)  - g(x)\right).
\end{equation}

Because the operator $F$ is monotone, we have
\begin{align}\label{alg1_eq2}
    \left\langle F(x), x_i - x \right\rangle & = \left\langle F(x_i), x_i - x   \right\rangle - \underbrace{\left\langle F(x_i) - F(x), x_i - x    \right\rangle}_{\geq 0} \nonumber
    \\& \leq \left\langle F(x_i), x_i - x   \right\rangle
\end{align}

We also have 
\begin{align} \label{alg1_eq3}
      h^F \sum_{i \in I} \left\langle F(x), x_i - x\right \rangle & =  h^F \left\langle F(x), \sum_{i \in I}   (x_i - x) \right \rangle = h^F |I| \left\langle F(x), \frac{1}{|I|} \sum_{i \in I}  x_i - x \right \rangle \nonumber
    \\& = h^F |I| \left\langle F(x),   \widehat{x} - x \right \rangle.
\end{align}

Therefore, from \eqref{alg1_eq1}, \eqref{alg1_eq2} and \eqref{alg1_eq3},  we get the following
\begin{align*}
    h^F |I| \left\langle F(x), \widehat{x} - x \right \rangle \leq \frac{\varepsilon^2}{2 L_F^2} |I| + \frac{\varepsilon^2}{2 M_g^2} |J| + R^2 - h^g \sum_{i \in J} \left( g(x_i)  - g(x)\right).
\end{align*}

Since for any $i \in J$, we have $g(x_i) - g(x_*) \geq g(x_i) > \varepsilon$, where $x_*$ is a solution of the problem under consideration, thus it represents a feasible point, i.e., $x_* \in Q$ and $g(x_*) \leq 0 $. Then, from the convexity of the function $g$, we have
\begin{align*}
    - h^g \sum_{i \in J} \left( g(x_i)  - g(x)\right) & = - h^g \sum_{i \in J} \left( g(x_i)  - g(x_*)\right) +  h^g \sum_{i \in J} \left( g(x)  - g(x_*)\right) 
    \\& < - h^g \sum_{i \in J} \varepsilon + h^g \sum_{i \in J} \left\langle \nabla g(x), x - x_* \right\rangle   
    \\& \leq - h^g  \varepsilon |J| + h^g M_g D |J|, 
\end{align*}
where in the last we used the Cauchy-Schwarz inequality and the fact that $g$ is $M_g$-Lipschitz, i.e., $\|\nabla g(x)\|_* \leq M_g, \;\; \forall x \in Q$. 

Thus, we get the following
\begin{align}
    h^F |I| \left\langle F(x),   \widehat{x} - x \right \rangle & <   \frac{\varepsilon^2}{2 L_F^2} |I| + \frac{\varepsilon^2}{2 M_g^2} |J|- h^g  \varepsilon |J| + h^g M_g D |J| + R^2 \nonumber
    \\& =  \frac{\varepsilon^2}{ L_F^2} |I|  - \frac{\varepsilon^2}{2 L_F^2} |I| - \frac{\varepsilon^2}{2 M_g^2} |J| + \frac{\varepsilon D |J|}{M_g} + R^2   \nonumber
    \\& =  \frac{\varepsilon^2}{ L_F^2} |I| - \left(\frac{\varepsilon^2}{2 L_F^2} |I| + \frac{\varepsilon^2}{2 M_g^2} |J| - \frac{\varepsilon D |J|}{M_g} \right) + R^2.   \label{eq1fgf2f1}
\end{align}

\begin{itemize}
\item From the first stopping rule (i.e., from \eqref{stop1_alg1}), we get 
\begin{equation*}
    h^F |I| \left\langle F(x),   \widehat{x} - x \right \rangle < \frac{\varepsilon^2}{ L_F^2} |I| \quad \forall x \in Q. 
\end{equation*}
Thus, since $h^F = \frac{\varepsilon}{L_F^2}$, we find 
\begin{equation*}
    \left\langle F(x),   \widehat{x} - x \right \rangle < \varepsilon \quad \forall x \in Q. 
\end{equation*}
\item   From the second stopping rule (i.e.,  from \eqref{stop2_alg1}), we get 
\begin{equation*}
    h^F |I| \left\langle F(x),   \widehat{x} - x \right \rangle < \frac{\varepsilon^2}{ L_F^2} |I| + \frac{\varepsilon D |J|}{M_g}\quad \forall x \in Q. 
\end{equation*}
Thus, since $h^F = \frac{\varepsilon}{L_F^2}$, we find 
\begin{equation*}
    \left\langle F(x),   \widehat{x} - x \right \rangle < \varepsilon + \frac{ D L_F^2|J|}{M_g |I|} \quad \forall x \in Q. 
\end{equation*}
\end{itemize}

Now, for any $i \in I$, we have $g(x_i) \leq \varepsilon$. Thus, from the convexity of $g$, we have 
\begin{equation*}
    g\left(\sum_{i \in I} x_i\right) \leq \sum_{i \in I} g(x_i) \leq \varepsilon |I|. 
\end{equation*}
This means that
$$
\frac{1}{|I|} g\left(\sum_{i \in I} x_i\right) \leq \varepsilon \quad \Longrightarrow \quad g\left(\frac{1}{|I|}\sum_{i \in I} x_i\right) = g(\widehat{x}) \leq \varepsilon. 
$$
\end{proof}

\subsection{Proof Theorem \ref{main_theorem_alg2}}\label{append_proof_theo_alg2}

\begin{proof}
For any $i \in I$, and $x \in Q$, we have 
\begin{align} \label{alg2_eq1_I}
    h_i^F \left\langle F(x_i), x_i - x\right \rangle & \leq \frac{(h_i^F)^2}{2} \|F(x_i)\|_*^2 + V(x, x_i) - V(x, x_{i+1}) \nonumber
    \\& =   \frac{h_i^F \varepsilon}{2}  + V(x, x_i) - V(x, x_{i+1}). 
\end{align}
For any $i \in J$, and $x \in Q$, we have 
\begin{align} \label{alg2_eq1_J}
    h_i^g \left( g(x_i)  - g(x)\right) & \leq \frac{(h_i^g)^2}{2} \|\nabla g(x_i)\|_*^2 + V(x, x_i) - V(x, x_{i+1}) \nonumber
    \\& = \frac{h_i^g \varepsilon}{2}  + V(x, x_i) - V(x, x_{i+1}) .
\end{align}

Summing up inequalities \eqref{alg2_eq1_I} and \eqref{alg2_eq1_J}, from $i = 0$ to $i = k-1$ for any $k \geq 1$, we get the following 
\begin{align*}\label{alg2_eq0}
    \sum_{i \in I} h_i^F \left\langle F(x_i), x_i - x\right \rangle +  \sum_{i \in J}h_i^g \left( g(x_i)  - g(x)\right)  
    & \leq \sum_{i \in I} \frac{h_i^F \varepsilon}{2} + \sum_{i \in J}\frac{h_i^g \varepsilon}{2}  
    \\& \quad +  \sum_{i = 0}^{k-1} \left(V(x, x_i) - V(x, x_{i+1})\right)  \nonumber
    \\& \leq \frac{\varepsilon}{2} \sum_{i = 0}^{k - 1} h_i + R^2,
\end{align*}
where in the last inequality, we set $h_i:= h_i^F$ for every $i \in I$ and $h_i : = h_i^g $ for every $i \in J$, and we used the fact
\begin{equation*}
    \sum_{i = 0}^{k-1} \left(V(x, x_i) - V(x, x_{i+1})\right) = V(x, x_0) - V(x, x_k) \leq V(x, x_0) \leq R^2. 
\end{equation*}

Thus, we have 
\begin{equation} \label{alg2_eq1}
     \sum_{i \in I}  h_i^F \left\langle F(x_i), x_i - x\right \rangle \leq \frac{\varepsilon}{2} \sum_{i = 0}^{k - 1} h_i + R^2 -  \sum_{i \in J} h_i^g \left( g(x_i)  - g(x)\right).
\end{equation}

Because the operator $F$ is monotone, we have
\begin{align}\label{alg2_eq2}
    \left\langle F(x), x_i - x \right\rangle & = \left\langle F(x_i), x_i - x   \right\rangle - \underbrace{\left\langle F(x_i) - F(x), x_i - x    \right\rangle}_{\geq 0} \nonumber
    \\& \leq \left\langle F(x_i), x_i - x   \right\rangle.
\end{align}

We also have 
\begin{align} \label{alg2_eq3}
    \sum_{i \in I} h_i^F  \left\langle F(x), x_i - x\right \rangle & =  \left\langle F(x), \sum_{i \in I} h_i^F (x_i - x) \right \rangle \nonumber
    \\& = \left(\sum_{i \in I} h_i^F\right)  \left\langle F(x), \frac{1}{\sum_{i \in I} h_i^F} \sum_{i \in I} h_i^F x_i - x \right \rangle \nonumber
    \\& = \left(\sum_{i \in I} h_i^F\right) \left\langle F(x),   \widehat{x} - x \right \rangle.
\end{align}

Therefore, from \eqref{alg2_eq1}, \eqref{alg2_eq2} and \eqref{alg2_eq3},  we get the following
\begin{align*}
    \left(\sum_{i \in I} h_i^F\right) \left\langle F(x),   \widehat{x} - x \right \rangle \leq \frac{\varepsilon}{2} \sum_{i = 0}^{k-1} h_i + R^2 -  \sum_{i \in J}h_i^g \left( g(x_i)  - g(x)\right).
\end{align*}

Since for any $i \in J$, we have $g(x_i) - g(x_*) \geq g(x_i) > \varepsilon$. 
Then by the convexity of the function $g$, we have
\begin{align*}
    -\sum_{i \in J}  h_i^g \left( g(x_i)  - g(x)\right) & = -  \sum_{i \in J} h_i^g \left( g(x_i)  - g(x_*)\right) +  \sum_{i \in J} h_i^g  \left( g(x)  - g(x_*)\right)
    \\& < -\varepsilon \sum_{i \in J}  h_i^g   + \sum_{i \in J}  h_i^g  \left\langle \nabla g(x), x - x_* \right\rangle
    \\& \leq -\varepsilon \sum_{i \in J}  h_i^g + M_g D \sum_{i \in J}  h_i^g, 
\end{align*}
where in the last we used the Cauchy-Schwarz inequality and the fact that $g$ is $M_g$-Lipschitz, i.e., $\|\nabla g(x)\|_* \leq M_g, \;\forall x \in Q$. 

Thus, we get the following
\begin{align*}
    & \quad \left(\sum_{i \in I} h_i^F\right) \left\langle F(x),   \widehat{x} - x \right \rangle   
    <   \frac{\varepsilon}{2} \sum_{i = 0}^{k-1} h_i -\varepsilon \sum_{i \in J}  h_i^g + M_g D \sum_{i \in J}  h_i^g + R^2
    \\& =  \frac{\varepsilon}{2} \sum_{i \in I} h_i^F + \frac{\varepsilon}{2}\sum_{i \in J} h_i^g -\varepsilon \sum_{i \in J}  h_i^g + M_g D \sum_{i \in J}  h_i^g + R^2 
    \\& =  \varepsilon \sum_{i \in I} h_i^F - \frac{\varepsilon}{2}   \sum_{i \in I} h_i^F - \frac{\varepsilon}{2}\sum_{i \in J} h_i^g + M_g D \sum_{i \in J}  h_i^g + R^2  
    \\& = \varepsilon \sum_{i \in I} h_i^F - \frac{\varepsilon^2}{2} \sum_{i \in I} \frac{1}{\|F(x_i)\|_*^2}  - \frac{\varepsilon^2}{2} \sum_{i \in J} \frac{1}{\|\nabla g(x_i)\|_*^2} +  M_g D \varepsilon \sum_{i \in J}\frac{1}{\|\nabla g(x_i)\|_*^2} 
    \\& \quad + R^2
    \\& = \varepsilon \sum_{i \in I} h_i^F - \left(\frac{\varepsilon^2}{2} \sum_{i = 0}^{k - 1} \frac{1}{M_i} -  M_g D \varepsilon \sum_{i \in J}\frac{1}{\|\nabla g(x_i)\|_*^2} \right) + R^2. 
\end{align*}
In the last equality, we used items 5 and 9 from the listing of Algorithm \ref{alg_2}.  Thus, we have 
\begin{align*}\label{alg2_eq_impo}
    \left(\sum_{i \in I} h_i^F\right) \left\langle F(x),   \widehat{x} - x \right \rangle   \nonumber 
    & < \varepsilon \sum_{i \in I} h_i^F - \left(\frac{\varepsilon^2}{2} \sum_{i = 0}^{k - 1} \frac{1}{M_i} -  M_g D \varepsilon \sum_{i \in J}\frac{1}{\|\nabla g(x_i)\|_*^2} \right)
    \\& \quad + R^2.
\end{align*}

\begin{itemize}
\item From the first stopping rule (i.e., from \eqref{stop1_alg2}), we get 
\begin{equation*}
    \left(\sum_{i \in I} h_i^F\right) \left\langle F(x),   \widehat{x} - x \right \rangle < \varepsilon \sum_{i \in I} h_i^F \quad \forall x \in Q. 
\end{equation*}
i.e.,  
\begin{equation*}
    \left\langle F(x),   \widehat{x} - x \right \rangle < \varepsilon \quad \forall x \in Q. 
\end{equation*}
    
\item   From the second stopping rule (i.e.,  from \eqref{stop2_alg2}), for any $x \in Q$,  we get 
\begin{equation*}
    \left(\sum_{i \in I} h_i^F\right) \left\langle F(x),   \widehat{x} - x \right \rangle <  \varepsilon \sum_{i \in I} h_i^F + M_g D \varepsilon \sum_{i \in J} \frac{1}{\|\nabla g(x_i)\|_*^2}. 
\end{equation*}
Thus, since $h_i^F = \frac{\varepsilon}{\|F(x_i)\|_*^2}$, we find 
\begin{equation*}
    \left\langle F(x),   \widehat{x} - x \right \rangle < \varepsilon +   M_g D\sum_{i \in J} \frac{1}{\|\nabla g(x_i)\|_*^2} \left( \sum_{i \in I} \frac{1}{\|F(x_i)\|_*^2}  \right)^{-1}.
\end{equation*}
\end{itemize}

Now, for any $i \in I$, we have $g(x_i) \leq \varepsilon$. Thus, from the convexity of $g$, we have 
\begin{align*}
    \left(\sum_{i \in I} h_i^F \right) g (\widehat{x}) & = \left(\sum_{i \in I} h_i^F \right) g \left( \frac{1}{\sum_{i \in I} h_i^F}\sum_{i \in I} h_i^F x_i \right) 
    \\& \leq \sum_{i \in I} h_i^F g(x_i) \leq  \sum_{i \in I} h_i^F \varepsilon. 
\end{align*}
This means, 
$$
    g(\widehat{x}) \leq \varepsilon. 
$$
\end{proof}

\subsection{Proof Theorem \ref{main_theorem_alg3}}\label{append_proof_theo_alg3}

\begin{proof}
For any $i \in I$, and $x \in Q$, we have 
\begin{align} \label{alg3_eq1_I}
    h_i^F \left\langle F(x_i), x_i - x\right \rangle & \leq \frac{(h_i^F)^2}{2} \|F(x_i)\|_*^2 + V(x, x_i) - V(x, x_{i+1}) \nonumber
    \\& =   \frac{\varepsilon^2}{2 \|F(x_i)\|_*^2}  + V(x, x_i) - V(x, x_{i+1}). 
\end{align}
For any $i \in J$, and $x \in Q$, we have 
\begin{align*} 
    h_i^g \left( g(x_i)  - g(x)\right) & \leq \frac{(h_i^g)^2}{2} \|\nabla g(x_i)\|_*^2 + V(x, x_i) - V(x, x_{i+1}) 
    \\& \leq \frac{(h_i^g)^2 }{2}M_g^2  + V(x, x_i) - V(x, x_{i+1}) 
    \\& = \frac{\varepsilon^2}{2}  + V(x, x_i) - V(x, x_{i+1}). 
\end{align*}

Thus, for any $i \in J$, we have
\begin{equation*}
     \frac{\varepsilon}{M_g} \left( g(x_i)  - g(x_*)\right) +  \frac{\varepsilon}{M_g} \left( g(x_*)  - g(x)\right) \leq \frac{\varepsilon^2}{2}  + V(x, x_i) - V(x, x_{i+1}). 
\end{equation*}

For any $i \in J$, we have $g(x_i) - g(x_*) \geq g(x_i) > \varepsilon M_g $, i.e., $\frac{g(x_i) - g(x_*)}{M_g} > \varepsilon$. Thus, we get the following
\begin{align*}
    \frac{\varepsilon^2}{2} & < V(x, x_i) - V(x, x_{i+1}) + \frac{\varepsilon}{M_g} \left( g(x) - g(x_*) \right)
    \\& \stackrel{(a)}{\leq}  V(x, x_i) - V(x, x_{i+1}) + \frac{\varepsilon}{M_g} \left\langle \nabla g(x), x - x_* \right \rangle
    \\& \stackrel{(b)}{\leq} V(x, x_i) - V(x, x_{i+1}) + \frac{\varepsilon}{M_g} \|\nabla g(x)\|_* \cdot \|x - x_*\|
    \\& \stackrel{(c)}{\leq} V(x, x_i) - V(x, x_{i+1}) + \varepsilon D,
\end{align*}
where in (a) we used the convexity of $g$, in (b) we used the Cauchy-Schwarz inequality, and in (c) the fact that $g$ is $M_g$-Lipschitz and $Q$ is bounded, its diameter is $D$. Thus, we get 
\begin{equation}\label{alg3_eq1_J}
    \frac{\varepsilon^2}{2} < V(x, x_i) - V(x, x_{i+1}) + \varepsilon D. 
\end{equation}

Summing up inequalities \eqref{alg3_eq1_I} and \eqref{alg3_eq1_J}, from $i = 0$ to $i = k-1$ for any $k \geq 1$, we get the following 
\begin{align}\label{eq_120}
     \sum_{i \in I} h_i^F \left\langle F(x_i), x_i - x\right \rangle   
     & < \frac{\varepsilon^2}{2} \sum_{i \in I} \frac{1}{\|F(x_i)\|_*^2}  + \sum_{i = 0}^{k-1} \left(V(x, x_i) - V(x, x_{i+1})\right)  \nonumber
     \\& \quad - \frac{\varepsilon^2}{2} |J| + \varepsilon D|J|   \nonumber
    \\& \leq \frac{ \varepsilon}{2} \sum_{i \in I}  h_i^F - \frac{\varepsilon^2}{2} |J| + \varepsilon D|J| + R^2 \nonumber
    \\& = \varepsilon \sum_{i \in I}  h_i^F - \frac{ \varepsilon}{2} \sum_{i \in I}  h_i^F  - \frac{\varepsilon^2}{2} |J| + \varepsilon D|J| + R^2 \nonumber
    \\& = \varepsilon \sum_{i \in I}  h_i^F - \left( \frac{\varepsilon^2}{2} \sum_{i \in I} \frac{1}{\|F(x_i)\|_*^2} +\frac{\varepsilon^2}{2} |J| - \varepsilon D|J| \right) + R^2. 
\end{align}

But, since $F$ is monotone, then in a similar way as in the proof of Theorem \ref{main_theorem_alg2}, we have 
\begin{equation}\label{eq111}
    \left(\sum_{i \in I} h_i^F\right) \langle F(x), \widehat{x} - x \rangle \leq \sum_{i \in I} h_i^F \langle F(x_i), x_i - x \rangle .
\end{equation}

Thus, from \eqref{eq_120} and \eqref{eq111} we get 
\begin{equation*}
     \left(\sum_{i \in I} h_i^F\right) \langle F(x), \widehat{x} - x \rangle < \varepsilon \sum_{i \in I}  h_i^F - \left( \frac{\varepsilon^2}{2} \sum_{i \in I} \frac{1}{\|F(x_i)\|_*^2} +\frac{\varepsilon^2}{2} |J| - \varepsilon D|J| \right) + R^2. 
\end{equation*}

\begin{itemize}
\item From the first stopping rule of Algorithm 3, we get 
\begin{equation*}
    \left(\sum_{i \in I} h_i^F\right) \left\langle F(x),   \widehat{x} - x \right \rangle < \varepsilon \sum_{i \in I} h_i^F \quad \forall x \in Q . 
\end{equation*}
i.e.,  
\begin{equation*}
    \left\langle F(x),   \widehat{x} - x \right \rangle < \varepsilon \quad \forall x \in Q. 
\end{equation*}
    
\item   From the second stopping rule of Algorithm 3, we get 
\begin{equation*}
    \left(\sum_{i \in I} h_i^F\right) \left\langle F(x),   \widehat{x} - x \right \rangle <  \varepsilon \sum_{i \in I} h_i^F + \varepsilon D |J| \quad \forall x \in Q . 
\end{equation*}
Thus, since $h_i^F = \frac{\varepsilon}{\|F(x_i)\|_*^2}$, we find 
\begin{equation*}
    \left\langle F(x),   \widehat{x} - x \right \rangle < \varepsilon +    D|J|\left( \sum_{i \in I} \frac{1}{\|F(x_i)\|_*^2}  \right)^{-1} \quad \forall x \in Q.
\end{equation*}
\end{itemize}

Now, for any $i \in I$, we have $g(x_i) \leq \varepsilon M_g$. Thus, from the convexity of $g$, we have 
\begin{align*}
    g (\widehat{x}) \leq \left(\sum_{i \in I} h_i^F \right)^{-1} \sum_{i \in I} h_i^F g(x_i) \leq \varepsilon M_g. 
\end{align*}
\end{proof}

\subsection{Proof Theorem \ref{main_theorem_alg4}}\label{append_proof_theo_alg4}

\begin{proof}
For any $i \in I$, and $x \in Q$, we have 
\begin{align} \label{alg4_eq1_I}
    h_i^F \left\langle F(x_i), x_i - x\right \rangle & \leq \frac{(h_i^F)^2}{2} \|F(x_i)\|_*^2 + V(x, x_i) - V(x, x_{i+1}) \nonumber
    \\& =   \frac{\varepsilon^2}{2}  + V(x, x_i) - V(x, x_{i+1}). 
\end{align}
For any $i \in J$, and $x \in Q$, we have 
\begin{align} \label{alg4_eq1_J}
    h_i^g \left( g(x_i)  - g(x)\right) & \leq \frac{(h_i^g)^2}{2} \|\nabla g(x_i)\|_*^2 + V(x, x_i) - V(x, x_{i+1}) \nonumber
    \\& = \frac{\varepsilon^2}{2 \|\nabla g(x_i)\|_*^2}  + V(x, x_i) - V(x, x_{i+1}) .
\end{align}

Summing up inequalities \eqref{alg4_eq1_I} and \eqref{alg4_eq1_J}, from $i = 0$ to $i = k-1$ for any $k \geq 1$, we get the following 
\begin{align*}
    \sum_{i \in I} h_i^F \left\langle F(x_i), x_i - x\right \rangle +  \sum_{i \in J}h_i^g \left( g(x_i)  - g(x)\right) 
     \leq  \frac{\varepsilon^2}{2} |I| + \sum_{i \in J}\frac{\varepsilon^2}{2 \|\nabla g(x_i)\|_*^2}  +  R^2,
\end{align*}
where we used the fact
\begin{equation*}
    \sum_{i = 0}^{k-1} \left(V(x, x_i) - V(x, x_{i+1})\right) = V(x, x_0) - V(x, x_k) \leq V(x, x_0) \leq R^2. 
\end{equation*}

But, since $F$ is monotone, then in a similar way as in the proof of Theorem \ref{main_theorem_alg2}, we have 
\begin{equation}\label{eq000}
    \left(\sum_{i \in I} h_i^F\right) \langle F(x), \widehat{x} - x \rangle \leq \sum_{i \in I} h_i^F \langle F(x_i), x_i - x \rangle .
\end{equation}

Therefore,  we get the following
\begin{align*}
    \left(\sum_{i \in I} h_i^F\right) \left\langle F(x),   \widehat{x} - x \right \rangle \leq \frac{\varepsilon^2}{2} |I| + \sum_{i \in J}\frac{\varepsilon^2}{2 \|\nabla g(x_i)\|_*^2}  +  R^2 - \sum_{i \in J}h_i^g \left( g(x_i)  - g(x)\right) .
\end{align*}

Since for any $i \in J$, we have $g(x_i) - g(x_*) \geq g(x_i) > \varepsilon$. Then by the convexity of the function $g$, we have
\begin{align*}
    -\sum_{i \in J}  h_i^g \left( g(x_i)  - g(x)\right) & = -  \sum_{i \in J} h_i^g \left( g(x_i)  - g(x_*)\right) +  \sum_{i \in J} h_i^g  \left( g(x)  - g(x_*)\right)
    \\& < -\varepsilon \sum_{i \in J}  h_i^g   + \sum_{i \in J}  h_i^g  \left\langle \nabla g(x), x - x_* \right\rangle
    \\& \leq - \sum_{i \in J}  \frac{\varepsilon^2}{\|\nabla g(x_i)\|_*^2}  + \varepsilon M_g D \sum_{i \in J} \frac{1}{\|\nabla g(x_i)\|_*^2}, 
\end{align*}
where in the last we used the Cauchy-Schwarz inequality and the fact that $g$ is $M_g$-Lipschitz, i.e., $\|\nabla g(x)\|_* \leq M_g, \;\; \forall x \in Q$. 

Thus, we get the following
\begin{align*}
    & \qquad \left(\sum_{i \in I} h_i^F\right) \left\langle F(x),   \widehat{x} - x \right \rangle   
    \\& <   \frac{\varepsilon^2}{2} |I| + \sum_{i \in J}  \frac{\varepsilon^2}{2 \|\nabla g(x_i)\|_*^2} +R^2 - \sum_{i \in J}  \frac{\varepsilon^2}{\|\nabla g(x_i)\|_*^2} + \varepsilon M_g D \sum_{i \in J} \frac{1}{\|\nabla g(x_i)\|_*^2}
    \\& =  \varepsilon^2 |I| - \frac{\varepsilon^2}{2} |I| - \frac{\varepsilon^2}{2} \sum_{i \in J}  \frac{1}{\|\nabla g(x_i)\|_*^2} + \varepsilon M_g D \sum_{i \in J} \frac{1}{\|\nabla g(x_i)\|_*^2} + R^2 
    \\& = \varepsilon^2 |I| - \left(\frac{\varepsilon^2}{2} |I| + \left(\frac{\varepsilon^2}{2}  - \varepsilon M_g D\right)  \sum_{i \in J} \frac{1}{\|\nabla g(x_i)\|_*^2}  \right)  +  R^2 .
\end{align*}

\begin{itemize}
\item From the first stopping rule of Algorithm 4, for any $x \in Q$,  we get 
\begin{equation*}
    \left(\sum_{i \in I} h_i^F\right) \left\langle F(x),   \widehat{x} - x \right \rangle < \varepsilon^2 |I|. 
\end{equation*}
Thus,  
\begin{equation*}
    \left\langle F(x),   \widehat{x} - x \right \rangle < \varepsilon^2 |I| \left(\sum_{i \in I} \frac{\varepsilon}{\|F(x_i)\|_*}\right)^{-1} = \varepsilon |I| \left(\sum_{i \in I} \frac{1}{\|F(x_i)\|_*}\right)^{-1}. 
\end{equation*}

But, $\|F(x_i)\|_* \leq L_F, \; \forall i \in I$, thus $\sum_{i \in I} \frac{1}{\|F(x_i)\|_*} \geq \frac{|I|}{L_F}$, and then we get the desired inequality, i.e., 
$$
    \left\langle F(x),   \widehat{x} - x \right \rangle < \varepsilon L_F \quad \forall x \in Q. 
$$
\item   From the second stopping rule of Algorithm 4, for any $x \in Q$, we get 
\begin{equation*}
    \left(\sum_{i \in I} h_i^F\right) \left\langle F(x),   \widehat{x} - x \right \rangle <  \varepsilon^2 |I| + M_g D \varepsilon \sum_{i \in J} \frac{1}{\|\nabla g(x_i)\|_*^2}. 
\end{equation*}
Thus, since $h_i^F = \frac{\varepsilon}{\|F(x_i)\|_*}$, and $\sum_{i \in I} \frac{1}{\|F(x_i)\|_*} \geq \frac{|I|}{L_F}$, we find the desired inequality, i.e.,
\begin{equation*}
    \left\langle F(x),   \widehat{x} - x \right \rangle < \varepsilon L_F +   \frac{M_g D  L_F}{|I|} \sum_{i \in J} \frac{1}{\|\nabla g(x_i)\|_*^2}. 
\end{equation*}
\end{itemize}

Now, for any $i \in I$, we have $g(x_i) \leq \varepsilon$. Thus, from the convexity of $g$, we have 
\begin{align*}
    \left(\sum_{i \in I} h_i^F \right) g (\widehat{x}) & = \left(\sum_{i \in I} h_i^F \right) g \left( \frac{1}{\sum_{i \in I} h_i^F}\sum_{i \in I} h_i^F x_i \right) 
    \\& \leq \sum_{i \in I} h_i^F g(x_i) \leq  \sum_{i \in I} h_i^F \varepsilon. 
\end{align*}
This means, 
$$
    g(\widehat{x}) \leq \varepsilon. 
$$   
\end{proof}

\subsection{Proof Theorem \ref{main_theorem_alg5}}\label{append_proof_theo_alg5}

\begin{proof}
For any $i \in I$, and $x \in Q$, we have 
\begin{align} \label{alg5_eq1_I}
    h_i^F \left\langle F(x_i), x_i - x\right \rangle & \leq \frac{(h_i^F)^2}{2} \|F(x_i)\|_*^2 + V(x, x_i) - V(x, x_{i+1}) \nonumber
    \\& =   \frac{\varepsilon^2}{2}  + V(x, x_i) - V(x, x_{i+1}). 
\end{align}
For any $i \in J, x \in Q$, and since $\|\nabla g(x_i) \|_* \leq M_g, \; \forall i \in J$,  we have 
\begin{align*} 
    h_i^g \left( g(x_i)  - g(x)\right) & \leq \frac{(h_i^g)^2}{2} \|\nabla g(x_i)\|_*^2 + V(x, x_i) - V(x, x_{i+1}) 
    \\& \leq  \frac{\varepsilon^2}{2}  + V(x, x_i) - V(x, x_{i+1}). 
\end{align*}

Thus, for any $i \in J$, we have
\begin{equation*}
     \frac{\varepsilon}{M_g} \left( g(x_i)  - g(x_*)\right) +  \frac{\varepsilon}{M_g} \left( g(x_*)  - g(x)\right) \leq \frac{\varepsilon^2}{2}  + V(x, x_i) - V(x, x_{i+1}). 
\end{equation*}

For any $i \in J$, we have $g(x_i) - g(x_*) \geq g(x_i) > \varepsilon M_g $, i.e., $\frac{g(x_i) - g(x_*)}{M_g} > \varepsilon$. Thus, as in the proof of Theorem \ref{main_theorem_alg3}, we get 
\begin{equation}\label{alg5_eq1_J}
    \frac{\varepsilon^2}{2} < V(x, x_i) - V(x, x_{i+1}) + \varepsilon D. 
\end{equation}

Summing up inequalities \eqref{alg5_eq1_I} and \eqref{alg5_eq1_J}, from $i = 0$ to $i = k-1$ for any $k \geq 1$, we get the following 
\begin{align}\label{eq_120_5}
     \sum_{i \in I} h_i^F \left\langle F(x_i), x_i - x\right \rangle   
     & < \frac{\varepsilon^2}{2} |I| - \frac{\varepsilon^2}{2} |J| + \varepsilon D|J| + \sum_{i = 0}^{k-1} \left(V(x, x_i) - V(x, x_{i+1})\right)    \nonumber
    \\& \leq \varepsilon^2 |I| - \left(\frac{\varepsilon^2}{2} |I|+ \frac{\varepsilon^2}{2} |J| - \varepsilon D |J| \right) + R^2. 
\end{align}

But, since $F$ is monotone, then in a similar way as in the proof of Theorem \ref{main_theorem_alg2}, we have 
\begin{equation}\label{eq111_5}
    \left(\sum_{i \in I} h_i^F\right) \langle F(x), \widehat{x} - x \rangle \leq \sum_{i \in I} h_i^F \langle F(x_i), x_i - x \rangle .
\end{equation}

Thus, from \eqref{eq_120_5} and \eqref{eq111_5} we get 
\begin{equation*}
     \left(\sum_{i \in I} h_i^F\right) \langle F(x), \widehat{x} - x \rangle < \varepsilon^2 |I| - \left(\frac{\varepsilon^2}{2} |I|+ \frac{\varepsilon^2}{2} |J| - \varepsilon D |J| \right) + R^2. 
\end{equation*}

\begin{itemize}
\item From the first stopping rule of Algorithm 5, for any $x \in Q$, we get 
\begin{equation*}
    \left(\sum_{i \in I} h_i^F\right) \left\langle F(x),   \widehat{x} - x \right \rangle < \varepsilon^2 |I|. 
\end{equation*}
Thus,  
\begin{equation*}
    \left\langle F(x),   \widehat{x} - x \right \rangle < \varepsilon^2 |I| \left(\sum_{i \in I} \frac{\varepsilon}{\|F(x_i)\|_*}\right)^{-1} = \varepsilon |I| \left(\sum_{i \in I} \frac{\varepsilon}{\|F(x_i)\|_*}\right)^{-1}. 
\end{equation*}

But, $\|F(x_i)\|_* \leq L_F, \; \forall i \in I$, thus $\sum_{i \in I} \frac{1}{\|F(x_i)\|_*} \geq \frac{|I|}{L_F}$, and then we get the desired inequality, i.e., 
$$
    \left\langle F(x),   \widehat{x} - x \right \rangle < \varepsilon L_F \quad \forall x \in Q. 
$$
    
\item From the second stopping rule of Algorithm 5, for any $x \in Q$,  we get 
\begin{equation*}
    \left(\sum_{i \in I} h_i^F\right) \left\langle F(x),   \widehat{x} - x \right \rangle <  \varepsilon^2 |I| +   \varepsilon D |J|. 
\end{equation*}
Thus, since $h_i^F = \frac{\varepsilon}{\|F(x_i)\|_*}$, and $\sum_{i \in I} \frac{1}{\|F(x_i)\|_*} \geq \frac{|I|}{L_F}$, we find the desired inequality, i.e.,
\begin{equation*}
    \left\langle F(x), \widehat{x} - x \right \rangle < \varepsilon L_F + \frac{ D L_F |J|}{|I|} \quad \forall x \in Q. 
\end{equation*}

\end{itemize}

Now, for any $i \in I$, we have $g(x_i) \leq \varepsilon M_g$. Thus, from the convexity of $g$, we have 
\begin{align*}
    g(\widehat{x}) \leq \left(\sum_{i \in I} h_i^F \right)^{-1} \sum_{i \in I} h_i^F g(x_i) \leq \varepsilon M_g. 
\end{align*}
\end{proof}

\subsection{Proof Theorem \ref{main_theorem_alg6}}\label{append_proof_theo_alg6}

\begin{proof}
For any $i \in I$, and $x \in Q$, we have 
\begin{align} \label{alg6_eq1_I}
    h_i^F \left\langle F(x_i), x_i - x\right \rangle & \leq \frac{(h_i^F)^2}{2} \|F(x_i)\|_*^2 + V(x, x_i) - V(x, x_{i+1}) \nonumber
    \\& = \frac{\varepsilon^2}{2 M_g^2}  + V(x, x_i) - V(x, x_{i+1}). 
\end{align}
For any $i \in J, x \in Q$, and since $\nabla g(x_i) \leq M_g, \; \forall i \in J$, we have 
\begin{align} \label{alg6_eq1_J}
    h_i^g \left( g(x_i)  - g(x)\right) & \leq \frac{(h_i^g)^2}{2} \|\nabla g(x_i)\|_*^2 + V(x, x_i) - V(x, x_{i+1}) \nonumber
    \\& \leq \frac{\varepsilon^2}{2 M_g^2}  + V(x, x_i) - V(x, x_{i+1}).
\end{align}

Summing up inequalities \eqref{alg6_eq1_I} and \eqref{alg6_eq1_J}, from $i = 0$ to $i = k-1$ for any $k \geq 1$, we get the following 
\begin{align*}
    \sum_{i \in I} h_i^F \left\langle F(x_i), x_i - x\right \rangle +  \sum_{i \in J}h_i^g \left( g(x_i)  - g(x)\right) 
     \leq  \frac{\varepsilon^2}{2 M_g^2}  (|I| + |J)| +  R^2,
\end{align*}
where we used the fact
\begin{equation*}
    \sum_{i = 0}^{k-1} \left(V(x, x_i) - V(x, x_{i+1})\right) = V(x, x_0) - V(x, x_k) \leq V(x, x_0) \leq R^2. 
\end{equation*}

But, since $F$ is monotone, then in a similar way as in the proof of Theorem \ref{main_theorem_alg2}, we have 
\begin{equation*}
    \left(\sum_{i \in I} h_i^F\right) \langle F(x), \widehat{x} - x \rangle \leq \sum_{i \in I} h_i^F \langle F(x_i), x_i - x \rangle .
\end{equation*}

Therefore,  we get the following
\begin{align*}
    \left(\sum_{i \in I} h_i^F\right) \left\langle F(x),   \widehat{x} - x \right \rangle \leq \frac{\varepsilon^2}{2M_g^2} (|I| + |J|) +  R^2 - \sum_{i \in J}h_i^g \left( g(x_i)  - g(x)\right) .
\end{align*}

Since for any $i \in J$, we have $g(x_i) - g(x_*) \geq g(x_i) > \varepsilon$. Then by the convexity of the function $g$, we have
\begin{align*}
    -\sum_{i \in J} h_i^g \left( g(x_i) - g(x)\right) & = -  \sum_{i \in J} h_i^g \left( g(x_i) - g(x_*)\right) +  \sum_{i \in J} h_i^g \left( g(x) - g(x_*)\right)
    \\& < -\varepsilon \sum_{i \in J} h_i^g   + \sum_{i \in J}  h_i^g  \left\langle \nabla g(x), x - x_* \right\rangle
    \\& \leq - \varepsilon \sum_{i \in J} \frac{\varepsilon }{M_g^2} + \frac{|J| \varepsilon}{M_g^2} \|\nabla g(x)\|_* \cdot \|x - x_*\|
    \\& \leq - \frac{\varepsilon^2}{M_g^2} |J| + \frac{\varepsilon D}{M_g} |J|, 
\end{align*}
where in the last two inequalities we used the Cauchy-Schwarz inequality and the fact that $g$ is $M_g$-Lipschitz, i.e., $\|\nabla g(x)\|_* \leq M_g, \;\; \forall x \in Q$. 

Thus, we get 
\begin{align*}
  \left(\sum_{i \in I} h_i^F\right) \left\langle F(x),   \widehat{x} - x \right \rangle   
  & <   \frac{\varepsilon^2}{2M_g^2} (|I| + |J|) - \frac{\varepsilon^2}{M_g^2} |J| + \frac{\varepsilon D}{M_g} |J| +  R^2
  \\& =\frac{\varepsilon^2}{M_g^2} |I| - \left( \frac{\varepsilon^2}{2M_g^2} |I| + \frac{\varepsilon^2}{2M_g^2} |J| -  \frac{\varepsilon D}{M_g} |J| \right) + R^2.
\end{align*}

\begin{itemize}
\item From the first stopping rule of Algorithm 6, for any $x \in Q$, we get 
\begin{equation*}
    \left(\sum_{i \in I} h_i^F\right) \left\langle F(x),   \widehat{x} - x \right \rangle < \frac{\varepsilon^2 |I|}{M_g^2}. 
\end{equation*}
Thus,  
\begin{equation*}
    \left\langle F(x), \widehat{x} - x \right \rangle < \frac{\varepsilon^2 |I|}{M_g^2} \left(\sum_{i \in I} \frac{\varepsilon}{\|F(x_i)\|_*}\right)^{-1} = \frac{\varepsilon^2 |I|}{M_g^2}  \left(\sum_{i \in I} \frac{1}{\|F(x_i)\|_*}\right)^{-1}. 
\end{equation*}

But, $\|F(x_i)\|_* \leq L_F, \; \forall i \in I$, thus $\sum_{i \in I} \frac{1}{\|F(x_i)\|_*} \geq \frac{|I|}{L_F}$, and then we get the desired inequality, i.e., 
$$
    \left\langle F(x),   \widehat{x} - x \right \rangle < \frac{\varepsilon L_F}{M_g}\quad \forall x \in Q. 
$$
\item   From the second stopping rule of Algorithm 6, for any $x \in Q$,  we get 
\begin{equation*}
    \left(\sum_{i \in I} h_i^F\right) \left\langle F(x),   \widehat{x} - x \right \rangle <  \frac{\varepsilon^2 |I|}{M_g^2} + \frac{\varepsilon |J| D}{M_g} . 
\end{equation*}
Thus, since $h_i^F = \frac{\varepsilon}{\|F(x_i)\|_*}$, and $\sum_{i \in I} \frac{1}{\|F(x_i)\|_*} \geq \frac{|I|}{L_F}$, we find the desired inequality, i.e.,
\begin{equation*}
    \left\langle F(x), \widehat{x} - x \right \rangle < \frac{\varepsilon L_F}{M_g} +  \frac{ |J| D L_F}{M_g |I|} \quad \forall x \in Q. 
\end{equation*}
\end{itemize}

Now, for any $i \in I$, we have $g(x_i) \leq \varepsilon$. Thus, from the convexity of $g$, we have 
\begin{align*}
    \left(\sum_{i \in I} h_i^F \right) g (\widehat{x}) & = \left(\sum_{i \in I} h_i^F \right) g \left( \frac{1}{\sum_{i \in I} h_i^F}\sum_{i \in I} h_i^F x_i \right) 
    \\& \leq \sum_{i \in I} h_i^F g(x_i) \leq  \sum_{i \in I} h_i^F \varepsilon. 
\end{align*}
This means, 
$$
    g(\widehat{x}) \leq \varepsilon. 
$$   
\end{proof}

\subsection{Proof Theorem \ref{main_theorem_alg7}}\label{append_proof_theo_alg7}

\begin{proof}
For any $i \in I$, and $x \in Q$, we have 
\begin{align} \label{alg7_eq1_I}
    \left\langle F(x_i), x_i - x\right \rangle \leq \frac{h_i^F}{2} \|F(x_i)\|_*^2 + \frac{1}{h_i^F} \left(V(x, x_i) - V(x, x_{i+1})\right).
\end{align}

For any $i \in J, x \in Q$, we have 
\begin{align} \label{alg7_eq1_J}
    g(x_i)  - g(x) \leq \frac{h_i^g}{2} \|\nabla g(x_i)\|_*^2 + \frac{1}{h_i^g} \left(V(x, x_i) - V(x, x_{i+1})\right). 
\end{align}

Summing up inequalities \eqref{alg7_eq1_I} and \eqref{alg7_eq1_J}, from $i = 0$ to $i = k-1$ for any $k \geq 1$, and by setting $h_i : = h_i^F \, \forall i \in I,  h_i : = h_i^g \, \forall i \in J$, we get the following 
\begin{align*}
    & \quad \sum_{i \in I} \left\langle F(x_i), x_i - x\right \rangle +  \sum_{i \in J} \left(g(x_i)  - g(x) \right)
    \\& 
    \leq \sum_{i = 0}^{k - 1} \frac{h_i M_i^2}{2} + \sum_{i = 0}^{k - 1} \frac{1}{h_i} \left(V(x, x_i) - V(x, x_{i+1})\right).
\end{align*}

But, since $F$ is monotone, then 
\begin{equation*}
     \langle F(x), x_i - x \rangle \leq  \langle F(x_i), x_i - x \rangle, \quad \forall i \in I, \, \forall x \in Q .
\end{equation*}
We also have, 
\begin{align*}
    \sum_{i \in I} \left\langle F(x), x_i - x\right \rangle & =  \left\langle F(x), \sum_{i \in I} (x_i - x) \right \rangle
    = |I| \left\langle F(x), \frac{1}{|I|} \sum_{i \in I} x_i - x \right \rangle
    \\& = |I| \left\langle F(x), \widehat{x} - x \right \rangle.
\end{align*}

Since for any $i \in J$, we have $g(x_i) - g(x_*) \geq g(x_i) > \varepsilon$. Then by the convexity of the function $g$, we have
\begin{align*}
    -\sum_{i \in J}  \left( g(x_i)  - g(x)\right) & = -  \sum_{i \in J}  \left( g(x_i)  - g(x_*)\right) +  \sum_{i \in J} \left( g(x)  - g(x_*)\right)
    \\& < -\varepsilon |J|  + |J|  \left\langle \nabla g(x), x - x_* \right\rangle
    \\& \leq - \varepsilon |J| + |J| \cdot  \|\nabla g(x)\|_* \cdot \|x - x_*\|
    \\& \leq - \varepsilon |J| + M_g D |J|, 
\end{align*}
In the last two inequalities, we used the Cauchy-Schwarz inequality and the fact that $g$ is $M_g$-Lipschitz, i.e., $\|\nabla g(x)\|_* \leq M_g, \; \forall x \in Q$. 

Thus, we get the following
\begin{align}\label{gdsfsd52}
    |I| \left\langle F(x), \widehat{x} - x \right \rangle  \nonumber
    & < \sum_{i = 0}^{k - 1} \frac{h_i M_i^2}{2}  + \sum_{i = 0}^{k - 1} \frac{1}{h_i} \left(V(x, x_i) - V(x, x_{i+1})\right) \nonumber
    \\& \quad + M_g D |J| - \varepsilon |J| \nonumber
    \\& = \frac{\theta}{2}\sum\limits_{i=0}^{k-1}  \frac{M_i^2}{\left(\sum_{t=0}^i M_t^2 \right)^{1/2}}  + \sum_{i = 0}^{k - 1} \frac{1}{h_i} \left(V(x, x_i) - V(x, x_{i+1})\right) 
    \\& \quad + M_g D |J| - \varepsilon |J|. \nonumber
\end{align}

But, 
\begin{align}\label{54412d}
    & \quad \sum\limits_{i = 0}^{k-1} \frac{1}{h_i} \left( V(x, x_i) - V(x, x_{i+1}) \right) \nonumber
    \\& = \frac{1}{h_0} V(x, x_0) +\sum_{i=0}^{k-2}  \left( \frac{1}{h_{i+1}} - \frac{1}{h_i} \right)  V(x, x_{i+1}) - \frac{1}{h_{k-1}} V(x,x_i) \nonumber 
    \\& \leq  \frac{\theta^2}{h_0} + \theta^2 \sum\limits_{i=0}^{k-2} \Big(\frac{1}{h_{i+1}} - \frac{1}{h_i} \Big) = \frac{\theta^2}{h_{k-1}} = \theta \left(\sum_{i = 0}^{k - 1} M_i^2\right)^{1/2}.
\end{align}

By induction (on $k \geq 1$), we can simply prove that
\begin{equation}\label{1201454}
    \sum\limits_{i=0}^{k-1} \frac{M_i^2}{\left(\sum_{t=0}^i M_t^2 \right)^{1/2}} \leq 2\left( \sum_{i=0}^{k-1} M_i^2 \right)^{1/2}.
\end{equation}

Therefore, from \eqref{gdsfsd52}, \eqref{54412d} and \eqref{1201454}, we get
\begin{align*}
    |I| \left\langle F(x), \widehat{x} - x \right \rangle <  2\theta \left( \sum_{i=0}^{k-1} M_i^2 \right)^{1/2} + |J| M _g D - \varepsilon k + \varepsilon |I|. 
\end{align*}

\begin{itemize}
\item From the first stopping rule of Algorithm 7, for any $x \in Q$, we get 
\begin{equation*}
    |I| \left\langle F(x), \widehat{x} - x \right \rangle < \varepsilon |I| \quad \Longrightarrow \quad  \left\langle F(x),   \widehat{x} - x \right \rangle < \varepsilon. 
\end{equation*}

\item   From the second stopping rule of Algorithm 7, for any $x \in Q$,  we get 
\begin{equation*}
    |I| \left\langle F(x), \widehat{x} - x \right \rangle <  \varepsilon |I| + |J| M _g D\quad \Longrightarrow \quad  \left\langle F(x),   \widehat{x} - x \right \rangle < \varepsilon + \frac{|J| M _g D}{|I|} . 
\end{equation*}

\end{itemize}

Now, for any $i \in I$, we have $g(x_i) \leq \varepsilon$. Thus, from the convexity of $g$, we have 
\begin{align*}
    |I| g (\widehat{x})  = |I| g \left( \frac{1}{|I|}\sum_{i \in I} x_i \right) 
    \leq \sum_{i \in I} g(x_i) \leq  |I| \varepsilon. 
\end{align*}
This means, 
$$
    g(\widehat{x}) \leq \varepsilon. 
$$   
\end{proof}

\section{Appendix 2: Proof Theorem \ref{main_theorem_alg2_modif}}\label{append_proof_theo_modalg2}

\begin{proof}
For $i \in I$, and $x \in Q$, we have 
\begin{align} \label{alg2_mod_eq1_I}
    h_i^F \left\langle F(x_i), x_i - x\right \rangle  \leq \frac{h_i^F \varepsilon}{2}  + V(x, x_i) - V(x, x_{i+1}). 
\end{align}
For $i \in  J$, and $x \in Q$, there is $N(i) \in \{1, 2, \ldots, m\}$ such that
\begin{align} \label{alg2_mod_eq1_J}
    h_i^{g_{N(i)}} \left( g_{N(i)}(x_i)  - g_{N(i)}(x)\right) & \leq \frac{\left(h_i^{g_{N(i)}}\right)^2}{2} \|\nabla g_{N(i)}(x_i)\|_*^2 + V(x, x_i) - V(x, x_{i+1}) \nonumber
    \\& = \frac{h_i^{g_{N(i)}} \varepsilon}{2}  + V(x, x_i) - V(x, x_{i+1}) .
\end{align}

Summing up inequalities \eqref{alg2_mod_eq1_I} and \eqref{alg2_mod_eq1_J}, from $i = 0$ to $i = k-1$ for any $k \geq 1$, we get the following 
\begin{equation*} 
     \sum_{i \in I}  h_i^F \left\langle F(x_i), x_i - x\right \rangle \leq \frac{\varepsilon}{2} \sum_{i = 0}^{k - 1} h_i + R^2 -  \sum_{i \in J} h_i^{g_{N(i)}} \left( g_{N(i)}(x_i)  - g_{N(i)}(x)\right).
\end{equation*}

Because of the operator $F$ is monotone, then we have
\begin{align*}
    \left(\sum_{i \in I} h_i^F\right) \left\langle F(x),   \widehat{x} - x \right \rangle \leq \frac{\varepsilon}{2} \sum_{i = 0}^{k-1} h_i + R^2 -  \sum_{i \in J} h_i^{g_{N(i)}} \left( g_{N(i)}(x_i)  - g_{N(i)}(x)\right).
\end{align*}

Since for $i \in J$, we have $g_{N(i)}(x_i) - g_{N(i)}(x_*) \geq g_{N(i)}(x_i) > \varepsilon$. Then by the convexity of the function $g_{N(i)}$, we have
\begin{align*}
    & \quad -\sum_{i \in J}  h_i^{g_{N(i)}} \left( g_{N(i)}(x_i)  - g_{N(i)}(x)\right) \\&  = -  \sum_{i \in J} h_i^{g_{N(i)}} \left( g_{N(i)}(x_i)  - g_{N(i)}(x_*)\right) +  \sum_{i \in J} h_i^{g_{N(i)}}  \left( g_{N(i)}(x)  - g_{N(i)}(x_*)\right)
    \\& < -\varepsilon \sum_{i \in J}  h_i^{g_{N(i)}}  + \sum_{i \in J}  h_i^{g_{N(i)}}  \left\langle \nabla g_{N(i)}(x), x - x_* \right\rangle
    \\& \leq -\varepsilon \sum_{i \in J}  h_i^{g_{N(i)}} + M_g D \sum_{i \in J}  h_i^{g_{N(i)}}, 
\end{align*}
where in the last we used the Cauchy-Schwarz inequality and the fact that $g_{N(i)}$ is $M_g$-Lipschitz, i.e., $\|\nabla g_{N(i)}(x)\|_* \leq M_g, \;\; \forall x \in Q$. 

Thus, we get the following
\begin{align*}
    & \quad \left(\sum_{i \in I} h_i^F\right) \left\langle F(x),   \widehat{x} - x \right \rangle <   \frac{\varepsilon}{2} \sum_{i = 0}^{k-1} h_i -\varepsilon \sum_{i \in J}  h_i^{g_{N(i)}} + M_g D \sum_{i \in J}  h_i^{g_{N(i)}} + R^2
    \\& =  \frac{\varepsilon}{2} \sum_{i \in I} h_i^F + \frac{\varepsilon}{2}\sum_{i \in J} h_i^{g_{N(i)}} -\varepsilon \sum_{i \in J}  h_i^{g_{N(i)}} + M_g D \sum_{i \in J}  h_i^{g_{N(i)}} + R^2 
    \\& =  \varepsilon \sum_{i \in I} h_i^F - \frac{\varepsilon}{2}   \sum_{i \in I} h_i^F - \frac{\varepsilon}{2}\sum_{i \in J} h_i^{g_{N(i)}} + M_g D \sum_{i \in J}  h_i^{g_{N(i)}} + R^2  
    \\& = \varepsilon \sum_{i \in I} h_i^F - \frac{\varepsilon^2}{2} \sum_{i \in I} \frac{1}{\|F(x_i)\|_*^2}  - \frac{\varepsilon^2}{2} \sum_{i \in J} \frac{1}{\|\nabla g_{N(i)}(x_i)\|_*^2}  \\& \quad+  M_g D \varepsilon \sum_{i \in J}\frac{1}{\|\nabla g_{N(i)}(x_i)\|_*^2} 
     + R^2
    \\& = \varepsilon \sum_{i \in I} h_i^F - \left(\frac{\varepsilon^2}{2} \sum_{i = 0}^{k - 1} \frac{1}{M_i} -  M_g D \varepsilon \sum_{i \in J}\frac{1}{\|\nabla g_{N(i)}(x_i)\|_*^2} \right) + R^2. 
\end{align*}

\bigskip 

\begin{itemize}
\item From the first stopping rule (i.e., from \eqref{stop1_alg2_mod}), we get 
\begin{equation*}
    \left(\sum_{i \in I} h_i^F\right) \left\langle F(x),   \widehat{x} - x \right \rangle < \varepsilon \sum_{i \in I} h_i^F . 
\end{equation*}
i.e.,  
\begin{equation*}
    \left\langle F(x),   \widehat{x} - x \right \rangle < \varepsilon. 
\end{equation*}
    
\item   From the second stopping rule (i.e.,  from \eqref{stop2_alg2_mod}), we get 
\begin{equation*}
    \left(\sum_{i \in I} h_i^F\right) \left\langle F(x),   \widehat{x} - x \right \rangle <  \varepsilon \sum_{i \in I} h_i^F + M_g D \varepsilon \sum_{i \in J} \frac{1}{\|\nabla g_{N(i)}(x_i)\|_*^2}. 
\end{equation*}
Thus, since $h_i^F = \frac{\varepsilon}{\|F(x_i)\|_*^2}$, we find 
\begin{equation*}
    \left\langle F(x),   \widehat{x} - x \right \rangle < \varepsilon +   M_g D\sum_{i \in J} \frac{1}{\|\nabla g_{N(i)}(x_i)\|_*^2} \left( \sum_{i \in I} \frac{1}{\|F(x_i)\|_*^2}  \right)^{-1}.
\end{equation*}
\end{itemize}

Now, for any $i \in I$, we have $g_j(x_i)   \leq \varepsilon, \; \forall j = 1, \ldots, m$, thus $g(x_i) = \max_{1 \leq j \leq m} \{g_j(x_i)\} \leq \varepsilon$. Therefore, from the convexity of $g$, we have 
\begin{align*}
    \left(\sum_{i \in I} h_i^F \right) g (\widehat{x}) & = \left(\sum_{i \in I} h_i^F \right) g \left( \frac{1}{\sum_{i \in I} h_i^F}\sum_{i \in I} h_i^F x_i \right) 
    \\& \leq \sum_{i \in I} h_i^F g(x_i) \leq  \sum_{i \in I} h_i^F \varepsilon. 
\end{align*}
This means $g(\widehat{x}) \leq \varepsilon$, thus
$$
    g_i(\widehat{x}) \leq \varepsilon, \quad \forall i \in \{1, 2, \ldots, m\}. 
$$
\end{proof}

\section{Appendix 3:  Analysis of the proposed algorithms with $\delta$-monotone operators}

\begin{definition}($\delta$-monotone operator). 
Let $\delta>0$. The operator $F :Q \longrightarrow \textbf{E}^*$ is called $\delta$-monotone, if it holds:
\begin{equation}\label{d_monot}
\langle F(y)-F(x),y-x\rangle \geq -\delta \quad \forall x, y \in Q.
\end{equation}
\end{definition}
For example, we can consider $F = \nabla_{\delta} f$ for $\delta$-subgradient $\nabla_{\delta} f(x)$ of convex function $f$ at point $x \in Q$: $f(y) - f(x) \geq \langle  \nabla_{\delta} f(x), y - x \rangle - \delta $ for each $y \in Q$ (see e.g., \cite{Polyak}, Chapter 5). Note that when $\delta = 0$, we get the definition of monotone operators. 

The analysis of algorithms under the assumption that the operator $F$ is $\delta$-monotone will be the same as the analysis with monotone operators, and the results will be the same with a slight difference, where we will have the following inequalities.

\begin{enumerate}
\item For Algorithm \ref{alg_1}, we have, 
\begin{itemize}
    \item with stopping criterion 1 \eqref{stop1_alg1}, we get 
\begin{equation}\label{alg1stop1_quality_F_sigma}
    \left\langle F(x), \widehat{x} - x \right\rangle < \varepsilon + \delta \quad \forall x \in Q .
\end{equation}

\item With stopping criterion 2 \eqref{stop2_alg1}, we get 
\begin{equation}\label{alg1stop2_quality_F_sigma}
    \left\langle F(x), \widehat{x} - x \right\rangle < \varepsilon  + \delta + \frac{D L_F^2 |J|}{M_g |I|}  \quad \forall x \in Q. 
\end{equation}
\end{itemize} 

\item For Algorithm \ref{alg_2}, we have,
\begin{itemize}
\item with stopping criterion 1 \eqref{stop1_alg2}, we get 
\begin{equation*}
    \left\langle F(x), \widehat{x} - x \right\rangle < \varepsilon + \delta  \quad \forall x \in Q.
\end{equation*}

\item With stopping criterion 2 \eqref{stop2_alg2}, we get 
\begin{equation*}
    \left\langle F(x), \widehat{x} - x \right\rangle < \varepsilon + \delta +  M_g D\sum_{i \in J} \frac{1}{\|\nabla g(x_i)\|_*^2} \left( \sum_{i \in I} \frac{1}{\|F(x_i)\|_*^2}  \right)^{-1} \quad \forall x \in Q. 
\end{equation*}

\end{itemize}

\item For Algorithm 3, we have,
\begin{itemize}
    \item with stopping criterion 1, we get 
\begin{equation*}
    \left\langle F(x), \widehat{x} - x \right\rangle < \varepsilon + \delta \quad \forall x \in Q.
\end{equation*}

\item With stopping criterion 2, we get 
\begin{equation*}
    \left\langle F(x), \widehat{x} - x \right\rangle < \varepsilon + \delta + D|J| \left(\sum_{i \in I} \frac{1}{\|F(x_i)\|_*^2}\right)^{-1} \quad \forall x \in Q.
\end{equation*}
\end{itemize}

\item For Algorithm 4, we have,
\begin{itemize}
\item with stopping criterion 1, we get 
\begin{equation*}
    \left\langle F(x), \widehat{x} - x \right\rangle < \varepsilon L_F + \delta  \quad \forall x \in Q.
\end{equation*}

\item With stopping criterion 2, we get 
\begin{equation*}
    \left\langle F(x), \widehat{x} - x \right\rangle < \varepsilon L_F + \delta + \frac{M_g D L_F}{|I|} \sum_{i \in I} \frac{1}{\|\nabla g(x_i)\|_*^2}  \quad \forall x \in Q.
\end{equation*} 
\end{itemize}

\item For Algorithm 5, we have,
\begin{itemize}
\item with stopping criterion 1, we get 
\begin{equation*}
    \left\langle F(x), \widehat{x} - x \right\rangle < \varepsilon L_F + \delta  \quad \forall x \in Q.
\end{equation*}

\item With stopping criterion 2, we get 
\begin{equation*}
    \left\langle F(x), \widehat{x} - x \right\rangle < \varepsilon L_F + \delta + \frac{ D L_F |J|}{|I|} \quad \forall x \in Q.
\end{equation*}
\end{itemize}

\item For Algorithm 6, we have,
\begin{itemize}
\item with stopping criterion 1, we get 
\begin{equation*}
    \left\langle F(x), \widehat{x} - x \right\rangle < \frac{\varepsilon L_F}{M_g} + \delta  \quad \forall x \in Q.
\end{equation*}

\item With stopping criterion 2, we get 
\begin{equation*}
    \left\langle F(x), \widehat{x} - x \right\rangle < \frac{\varepsilon L_F}{M_g} + \delta +  \frac{D L_F |J|}{M_g |I|} \quad \forall x \in Q.
\end{equation*}
\end{itemize}

\item For Algorithm 7, we have,
\begin{itemize}
\item with stopping criterion 1, we get 
\begin{equation*}
    \left\langle F(x), \widehat{x} - x \right\rangle < \varepsilon + \delta  \quad \forall x \in Q.
\end{equation*}

\item With stopping criterion 2, we get 
\begin{equation*}
    \left\langle F(x), \widehat{x} - x \right\rangle < \varepsilon  +  \frac{|J|M_g D  }{ |I|} + \delta \quad \forall x \in Q.
\end{equation*}
\end{itemize}

\end{enumerate}

Let us briefly show how we can prove the result of Algorithm \ref{alg_1} for $\delta$-monotone operator $F$, i.e., the results \eqref{alg1stop1_quality_F_sigma} and \eqref{alg1stop2_quality_F_sigma}. The proof will be the same as the proof of Theorem \ref{main_theorem_alg1}, with a slight modification concerning the $\delta$ monotonicity of the operator $F$. 

For any $i \in I$, and $x \in Q$, since $F$ is $\delta$-monotone operator, we get 
\[
    \left\langle F(x), x_i - x\right\rangle - \delta \leq \left\langle F(x_i), x_i - x \right\rangle. 
\]

Now, from \eqref{alg1_eq1}, \eqref{alg1_eq3} and \eqref{eq1fgf2f1}, for any $x \in Q$, we get 
\begin{align*}
    h^F |I| \left(\left\langle F(x),   \widehat{x} - x \right \rangle  -  \delta\right) < \frac{\varepsilon^2}{ L_F^2} |I| - \left(\frac{\varepsilon^2}{2 L_F^2} |I| + \frac{\varepsilon^2}{2 M_g^2} |J| - \frac{\varepsilon D |J|}{M_g} \right) + R^2. 
\end{align*}

Thus, with stopping criterion 1 \eqref{stop1_alg1} we get \eqref{alg1stop1_quality_F_sigma}, and with stopping criterion 2 \eqref{stop2_alg1} we get \eqref{alg1stop2_quality_F_sigma}.

\section{Appendix 4: Additional experiments: 2D Example for min-max problem: Forsaken Game}

In this section, we provide the following 2D example with min-max objectives and one ellipse constraint to illustrate the trajectory of Algorithms \ref{alg_2}---6. This problem is connected with the forsaken game \cite{Hsieh2021Limits}. It does not have constraints in its original formulation, and we add an ellipse constraint to test our algorithms \cite{Zhang2024Primal}. 

The forsaken game has the objective
\begin{equation}
\begin{aligned}
    & \min _{x \in \mathbb{R}} \max _{y \in \mathbb{R}} \left\{x(y-0.45)+h(x)-h(y) \right\}, \\
    & \text { s.t. } \quad x^2+4 y^2 \leq 1,
\end{aligned}
\end{equation}
where $h(x) = \frac{x^2}{4}  - \frac{x^4}{2} + \frac{x^6}{6}$. This problem has a desirable $(x_*, y_*) \approx (0.08, 0.4)$. This example is known to exhibit limit cycles, complicating equilibrium computations. Notably, the limit cycle near the solution is unstable, repelling any trajectories that approach it. 

\begin{figure}[htp]
    \minipage{0.50\textwidth}
    \includegraphics[width=\linewidth]{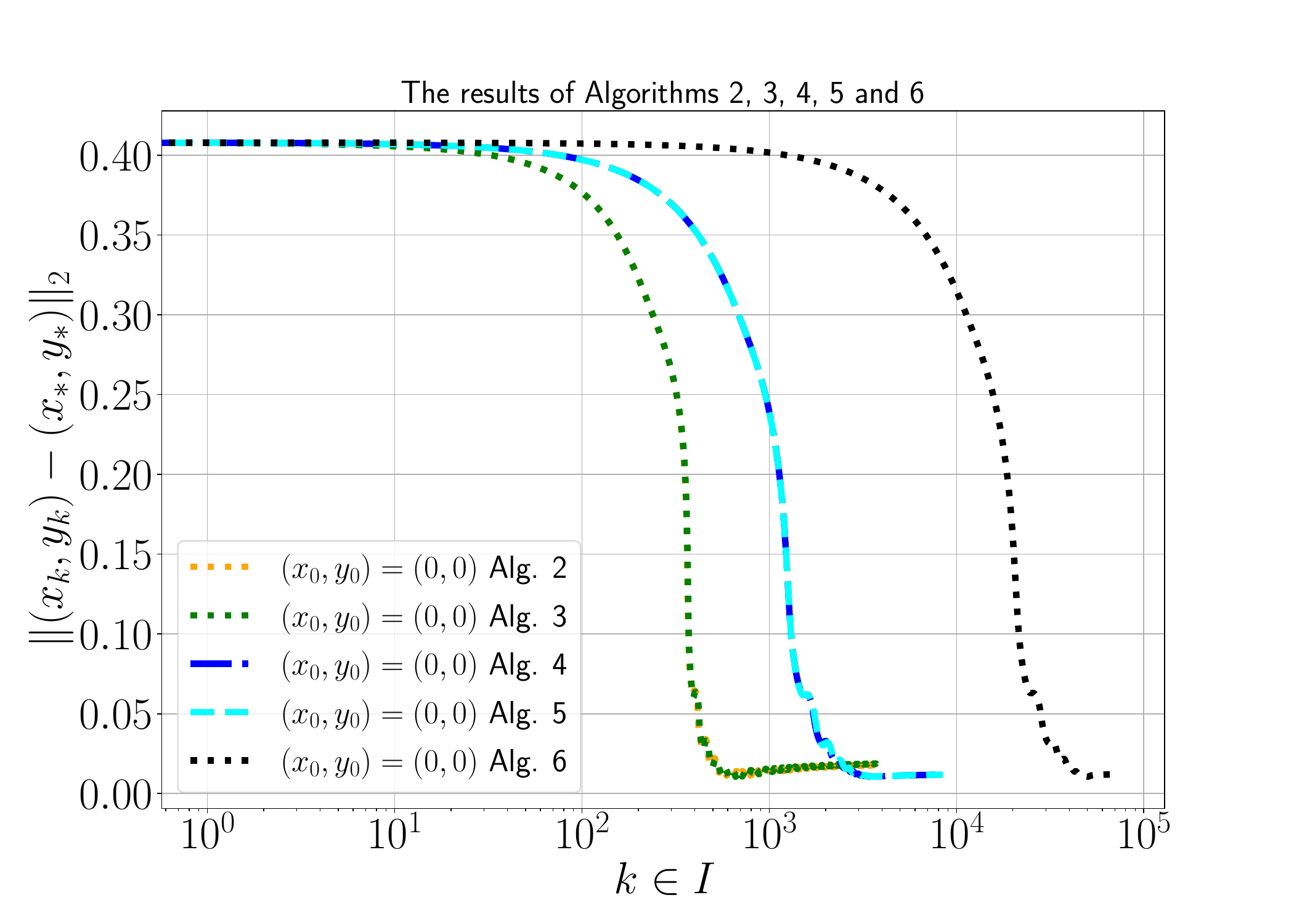}
    \endminipage
    \minipage{0.50\textwidth}
    \includegraphics[width=\linewidth]{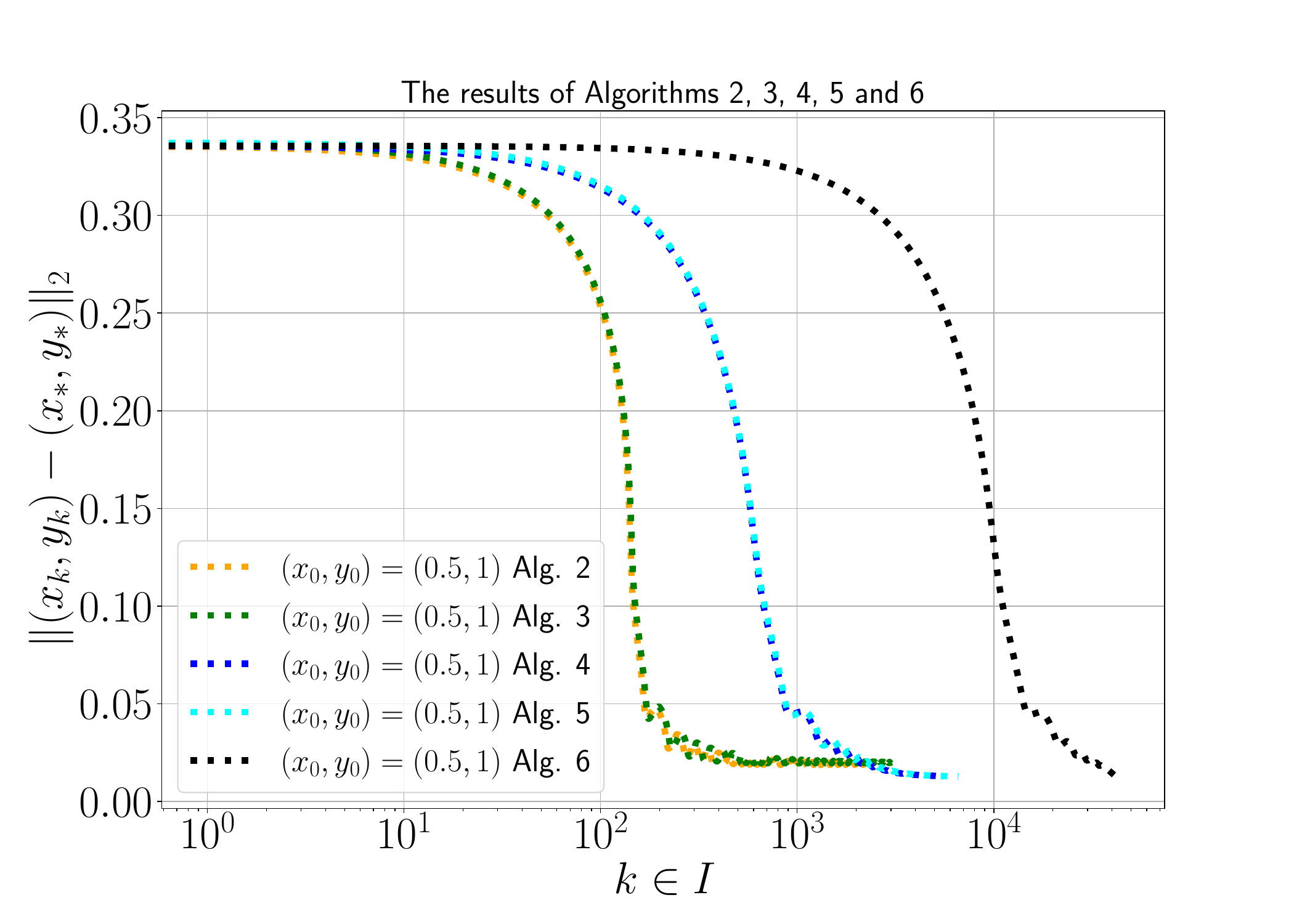}
    \endminipage
    \caption{Results of  Algorithms \ref{alg_2}---6, for the forsaken game with different initial points, with $\varepsilon = 0.001$ and by $10^4$ iterations of each algorithm, except Algorithm 6 with $10^5$ iterations.}
    \label{fig_smallest_covering_n200}
\end{figure} 

\end{document}